\documentclass[sn-mathphys]{sn-jnl}  
\usepackage{indentfirst} 
\usepackage{amsmath}
\usepackage{amsfonts}
\usepackage{amssymb}
\usepackage{multirow}
\usepackage{amsthm}
\usepackage{epsfig}
\usepackage{fancyhdr}
\usepackage{float}
\usepackage{graphicx}
\usepackage{CJK}
\usepackage{pdfpages}
\usepackage{framed}
\usepackage{algorithm}
\usepackage{algpseudocode}
\usepackage{amsmath}
\usepackage{graphics}
\usepackage{epsfig}
\usepackage{bm}
\usepackage{booktabs}
\usepackage[title]{appendix}
\usepackage{graphicx}  
\usepackage{subfigure} 
\usepackage{epstopdf}

\jyear{2022} 
\theoremstyle{thmstyleone}
\newtheorem{theorem}{Theorem}
\newtheorem{proposition}[theorem]{Proposition}
\theoremstyle{thmstyletwo}

\theoremstyle{thmstylethree}

\begin{document}
\newtheorem{lemma}{\bf Lemma\,}[section]
 
\title[GOLFS: Feature Selection for High Dimensional Clustering]{GOLFS: Feature Selection via Combining Both Global and Local Information for High Dimensional Clustering}  
 
\author[1]{\fnm{Zhaoyu } \sur{Xing}}  

\author[3]{\fnm{Yang } \sur{Wan}} 

\author*[2]{\fnm{Juan} \sur{Wen}} \email{wenjuan@xmu.edu.cn}
\author*[1,2]{\fnm{Wei } \sur{Zhong}} \email{wzhong@xmu.edu.cn}

\affil*[1]{Wang Yanan Institute for Studies in Economics, Xiamen University, China}
\affil*[2]{Department of Statistics and Data Science, School of Economics, Xiamen University, China}
\affil[3]{ByteDance Ltd., Beijing, China}

\abstract{It is important to identify the discriminative features for high dimensional clustering. However, due to the lack of cluster labels, the regularization methods developed for supervised feature selection can not be directly applied. To learn the pseudo labels and select the discriminative features simultaneously, we propose a new unsupervised feature selection method, named GlObal and Local information combined Feature Selection (GOLFS), for high dimensional clustering problems. The GOLFS algorithm combines both local geometric structure via manifold learning and global correlation structure of samples via regularized self-representation to select the discriminative features. The combination improves the accuracy of both feature selection and clustering by exploiting more comprehensive information. In addition, an iterative algorithm is proposed to solve the optimization problem and the convergency is proved. Simulations and two real data applications demonstrate the excellent finite-sample performance of GOLFS on both feature selection and clustering.}


\maketitle

\section{Introduction}
Massive high-dimensional data are collected in various fields of scientific research and empirical applications, such as text data \citep{L2005}, image data \citep{Yd2017} and bioinformatics data \citep{Ld2008}. The high dimensionality brings many computational and theoretical challenges to traditional statistical methods. To handle these challenges, the sparsity principle \citep{Hd2015} is often assumed that only a small subset of features in high-dimensional space is able to contribute to the response. Under the sparsity, feature selection plays a significant role in knowledge discovery and statistical modeling \citep{FL2006, ZL2012, Fd2013}. When the response is observed in the high-dimensional supervised learning, the regularization methods have been well studied for feature selection, such as LASSO \citep{T1996} and SCAD \citep{FL2001}.
They can simultaneously select important features and estimate their effects in high dimensional regression or classification via minimizing a penalized loss function. For the ultrahigh-dimensional problem where the dimension of features is much greater than the sample size, independent feature screening methods have also been proposed, such as SIS\citep{FL2008}, DC-SIS \citep{Ld2012}, MV-SIS \citep{CLZ2015} and among others. These methods rank the importance of features based on various marginal utilities between each feature and the response. For more information, one can refer to several excellent books on high-dimensional statistics \citep{HTF2009, FLZZ2020}.

However, the response information is not available in many unsupervised learning problems such as clustering. The aforementioned feature selection methods for supervised learning can not be directly applied. Due to the increasing popularity of unlabeled data, there are increasing demands for high-dimensional clustering in many applications, especially in the clustering of image \citep{Cd2007} and text \citep{AZ2012}. The “curse of dimensionality” leads to many challenges for clustering algorithms, and some dimensional reduction methods are directly used to solve the high-dimensional clustering problem by performing principal components analysis(PCA) \citep{Bernardo} and clustering the components. However, the clustering based on components is not sparse in features in practice. Some modifications of this approach are proposed such as Sparse PCA \citep{Zou2006} which obtains sparse loadings while constructing principal components, and IF-PCA \citep{JW2016AOS} which performs as a two-stage clustering by adding a feature selection procedure before classical PCA. Some statistical methods focus on solving the problem of subspace clustering which assumes the clusters can be distinguished on different attributes \citep{Hoff}, and perform the model-based clustering methods for high-dimensional cases \citep{Lim2021, Friedman2004, RD2006}. However, these approaches are shown to be sensitive to outliers in the dataset \citep{Ezugwu}, having a high computational cost for large datasets like gene expression microarrays \citep{Friedman2004} and does not truly result in a sparse clustering \citep{Tibshirani2010JASA}.  Another problem for high-dimensional clustering is the invalid distance measures. Classical clustering algorithms such as K-Means and hierarchical clustering are based on the distance measures between items. Due to many irrelevant features in high-dimensional data, these distance measures can hardly capture the true structure of data and become unreliable for clustering.  
Therefore, the unsupervised feature selection (UFS) methods that are able to identify the discriminative features without knowing the true class labels have raised more and more attention \citep{SFd2020} from the field of both data science and statistics. 

There are three types of UFS approaches i.e. filter method \citep{HCN2005}, wrapper method \citep{DB2004, BL2011}, and embedded method \citep{Yd2011, Ld2012u, DS2015, Wd2015}. The filter methods like the independence feature selection methods rank all features and select the top ones based on predefined evaluation criteria such as variance \citep{DB2004} and Laplace score \citep{HCN2005}. These methods are often computationally efficient, but the separation of feature selection and clustering can lead to an unstable performance. The wrapper methods use the evaluations of final clustering results to evaluate the feature subsets directly. These goal-oriented procedures have good performance but their computational costs are usually high, especially for large-scale and high-dimensional datasets \citep{Sd2019}. The embedded methods introduce the feature selection function into clustering procedures by a well-designed optimization problem which contains several terms for both clustering and feature selection. This class of methods performs clustering and feature selection simultaneously with acceptable computational cost.

Many embedded UFS methods use manifold learning theory to select the feature subset that best preserves the local structure of data. For these methods, spectral clustering based on $k$-adjacent graphs is commonly used to learn pseudo labels, and penalties are usually imposed to select candidate features. Multi-Cluster Feature Selection (MCFS) \citep{Cd2010} selects features by penalized regression methods with pseudo labels learned by spectral clustering directly.  Although MCFS is easy to apply, the simple combination of two separate processes makes the performance of feature selection highly dependent on the accuracy of the first-stage clustering. Unsupervised Discriminative Feature Selection (UDFS) \citep{Yd2011} assumes the labels of data can be predicted locally with a linear classifier and incorporates the local discriminative information to do feature selection in one stage. The Nonnegative Discriminative Feature Selection (NDFS) \citep{Ld2012u} imposes a positive constrain to the pseudo label matrix based on spectral clustering and shrinks the coefficients with the $l_{2,1}$-norm penalty term. The additional improvement with controlling the redundant features is proposed in \citep{Ld2015}. As these methods only focus on the local structure of data to learn pseudo labels, they ignore the global correlation structures of data. We are motivated to improve the learning process by adding correlation information in samples.

Another group of embedded UFS methods \citep{Liud2010, Sd2019, Liangd2018} aims to select the features that can represent and reconstruct the original feature space. The Regularized Self-Representation method (RSR) \citep{Zd2015} assumes the redundant features can be treated as the linear combinations of the relevant features, thus the penalties can be imposed on the coefficient matrix to select the representative features. This idea mainly comes from the low-rank recovery methods of sparse dictionary learning \citep{Liud2010}. The RSR method only considers the correlations and linear dependence among features and ignores the local structure of data. It can potentially lose some important features that are not correlated to other features but preserve discriminative information. We are motivated to exploit the idea of regularized self-representation to additionally incorporate the global correlation structure of data samples.

In this paper, we propose a new unsupervised feature selection method, named GlObal and Local information combined Feature Selection (GOLFS), for high dimensional spectral clustering. It can learn the pseudo labels and select the discriminative features simultaneously. The GOLFS algorithm combines both local geometric structure via manifold learning and global correlation structure of samples via regularized self-representation. Same as the NDFS in  \citep{Ld2012u}, the $l_{2,1}$-norm penalty is imposed on the coefficient matrix to select the discriminative features.
The combination improves the accuracy of both feature selection and clustering by exploiting more comprehensive information. We remark that we borrow the idea of the RSR to study the global structure of the representation in terms of samples, not features. An iterative algorithm is designed to solve our model efficiently and its convergence is studied.  Simulations and two real data applications demonstrate the excellent finite-sample performance of GOLFS on both feature selection and clustering.

The rest of the paper is organized as follows.  Section 2 introduces the GOLFS method as well as its iterative algorithm. The simulations are conducted in Section 3. The empirical studies of two real-world datasets are presented in section 4. We conclude with discussions in Section 5. The proofs are provided in the Appendix.

\section{Methodology of GOLFS}

\subsection{Preliminaries}
Throughout this paper, we denote the matrices, vectors, and scalars as bold uppercase letters, bold lowercase characters, and normal lowercase characters, respectively.
For any matrix $\mathbf{A}=(A_{ij})\in \mathbb{R}^{n\times m}$,
$Tr(\mathbf{A})$ denotes the trace of $\mathbf{A}$, the Frobenius norm and $l_{2,1}$-norm of  $\mathbf{A}$ are defined as $\left\|\mathbf{A}\right\|_F = \sqrt{\sum_{i=1}^n \sum_{j=1}^m a_{ij}^2}$ and $\left\| \mathbf{A}\right\|_{2,1} = \sum_{i=1}^n \sqrt{\sum_{j=1}^m a_{ij}^2}$, respectively. For any vector $\mathbf{v}=(v_{i})\in \mathbb{R}^{n\times 1}$,  $\|\mathbf{v}\|_1=\sum_{i}\lvert v_{i}\rvert$ and $\|\mathbf{v}\|_2=\sqrt{\sum_{i}v_{i}^2}$ represent the $L_1$-norm and $L_2$-norm of  $\mathbf{v}$, respectively.

Denote $\mathbf{X}\in \mathbb{R}^{n\times d}$ as the data matrix for $d$ features with the sample size $n$.
Let $\mathbf{x}_i\in \mathbb{R}^{1\times d}$ be a vector of observations of $d$ features for the $i$-th subject, that is, the $i$-th row of $\mathbf{X}$. Thus, $\mathbf{X}=[\mathbf{x}_1',\mathbf{x}_2',\ldots,\mathbf{x}_n']'$.
$X_{ij}$ stands for the element of $\mathbf{X}$ at the $i$-th row and the $j$-th column.
Suppose the $n$ observations come from $c$ clusters, we use indicator vectors $\mathbf{y}_i\in \{0,1\}^{1\times c}$ to show the cluster of the $i$th subject, and $\mathbf{Y}=[\mathbf{y}_1',\mathbf{y}_2',\cdot\cdot\cdot,\mathbf{y}_n']' \in \{0,1\}^{n\times c}$ is the indicator matrix.
Furthermore, we denote the scaled cluster indicator matrix $\mathbf{F}=[\mathbf{f}_1',\mathbf{f}_2',\cdot\cdot\cdot,\mathbf{f}_n']'=\mathbf{Y}(\mathbf{Y}'\mathbf{Y})^{-\frac{1}{2}}$ same as \cite{YSNJ2011}. It's easy to check $\mathbf{F}'\mathbf{F}=\mathbf{I}_c$, where $\mathbf{I}_c \in \mathbb{R}^{c\times c} $ is an identity matrix.

The embedded methods solve a regularized optimization problem of the indicator matrix $\mathbf{F}$ via introducing the row sparsity of the coefficient matrix $\mathbf{W}$ of $\mathbf{X}$ to perform clustering and feature selection simultaneously. The general form of the objective function for the embedded methods is
\begin{equation}
\label{GF}
\min_{\mathbf{F},\mathbf{W}} f(\mathbf{X},\mathbf{F}) + \alpha l(\mathbf{X},\mathbf{F},\mathbf{W}) + \beta R(\mathbf{W}),
\end{equation}
where the first term $f(\mathbf{X},\mathbf{F})$ is a loss function for clustering to learn the pseudo labels. For spectral clustering, $f(\mathbf{X},\mathbf{F})=Tr(\mathbf{F}'\mathbf{L}\mathbf{F})$, where $\mathbf{L}$ is the Laplace matrix generated from the $k$-adjacent matrix based on the pairwise distances among sample points. The second term $l(\mathbf{X},\mathbf{F},\mathbf{W})$ is the loss function
to connect the pseudo labels with features. For example, it is usually assumed that there is a linear relationship between features and the pseudo labels, leading to $l(\mathbf{X},\mathbf{F},\mathbf{W})=\left\| \mathbf{XW}-\mathbf{F}\right\|_F^2$, where $\mathbf{W}\in \mathbb{R}^{d\times c}$ is a coefficient matrix.
The third term $R(\mathbf{W})$ is the regularization term of $\mathbf{W}$ to encourage its row sparsity for feature selection. $\alpha$ and $\beta$ are tuning parameters
to control the weight between the first two loss functions and the magnitude of the regularization on the row sparsity of $\mathbf{W}$, respectively. Many existing embedded feature selection methods \citep{Yd2011, Ld2012u, Ld2015, Wd2015, Nd2016} are the special cases of this general form.

To better understand the objective function (\ref{GF}), we present the following lemma to establish the equivalence between the original algorithm of spectral clustering and an optimization problem. This lemma provides the foundation for the embedded feature selection methods. The proof of Lemma \ref{TH1} is shown in the Appendix. For more information about spectral clustering, one can refer to a tutorial in \cite{Lux2007}.
\begin{lemma}
\label{TH1}
The minimization of the cut loss in spectral clustering is equivalent to solving the following optimization problem,
\begin{eqnarray}
\arg \min_{\mathbf{F}} ~ Tr ( \mathbf{F}' \mathbf{LF}), ~~
\mbox{s.t.} ~ \mathbf{F}'\mathbf{F}=\mathbf{I},
\end{eqnarray}
where the $\mathbf{L}$ is the Laplace matrix corresponding to adjacency matrix $\mathbf{A}$ and $k$ is pre-defined numbers of clusters.
\end{lemma}

\subsection{The objective function of GOLFS}

The aforementioned embedded feature selection methods usually only utilized the Laplace matrix  $\mathbf{L}$ generated from the $k$-adjacent graph in spectral clustering to preserve the local structure of data. However, we observe that the global correlation structure among data samples
is able to improve the learning process of the pseudo labels and further enhance the accuracy of feature selection for spectral clustering. In this paper, we consider the idea of regularized self-representation to additionally incorporate the global correlation structure of data samples in the first term of the minimization problem (\ref{GF}).

Firstly, we solve the global structure of data samples through the following objective function:
\begin{equation}\label{global}
\min_{\mathbf{P}} \left\| \mathbf{X}'-\mathbf{X}'\mathbf{P}\right\|_{2,1}+\kappa\left\|\mathbf{P}\right\|_{2,1},
\end{equation}
where $\mathbf{P}\in \mathbb{R} ^{n\times n}$ is the coefficient matrix of $\mathbf{X}$ which  describes the self-representation information in the sample space, $\kappa \geq 0$ is the tuning parameter, $l_{2,1}$-norm encourages the row sparsity of $\mathbf{P}$. Though the objective function (\ref{global}) is convex, two terms are both non-smooth and an iterative reweighted least-squares algorithm is used to solve the optimal $\mathbf{P}$. The elements of $\mathbf{P}$ stand for the weights of edges between observations and show the global correlation structures of sample points.
That is, each sample point can be represented as the linear combination of its most relevant sample points.
Thus, $\mathbf{P}$ can also be treated as the special adjacency matrix for spectral clustering from a global correlation perspective. Here we focus on the strength of the connections between samples and ignore the sign of coefficients. So after we obtain the optimal estimated $\widehat{\mathbf{P}}$ of (\ref{global}), a symmetric global similarity matrix $\mathbf{S_1}$ is obtained by
\begin{equation}
\mathbf{S_1}=\frac{ \widehat{\mathbf{P}} \odot \mathbf{\Phi}    + ( \widehat{\mathbf{P}} \odot\mathbf{\Phi})'  }{2},
\end{equation}
where $\odot$ represents the calculation of Hadamard product and $\bm{\Phi}$ is the sign matrix of $\widehat{\mathbf{P}}$, that is $\bm{\Phi}_{ij}=sgn(P_{ij})$ for all $i$ and $j$. And the Laplace matrix $\mathbf{L}_1$ corresponding to the global similarity matrix $\mathbf{S}_1$ is given by
$\mathbf{L}_1=\mathbf{D}_1-\mathbf{S}_1,$
where $\mathbf{D}_1$ is a diagonal matrix with $D_{ii}=\sum_{j=1}^n (S_1)_{ij},~i=1,2,\cdot\cdot\cdot,n$.

Besides, to preserve the local geometric structure of data, we also consider the spectral clustering algorithm using manifold learning  \citep{Cd2010L, Ld2012u}. The $n$ observations are viewed as $n$ vertices in an undirected graph. Each point and its $k$ nearest neighbors are connected with weight $(S_0)_{ij}$ generated from the Gaussian kernel, that is,
\begin{equation}
\label{ }
(S_0)_{ij} =
\left\{
\begin{array}{ll}
\exp\left(-\frac{\left\|\mathbf{x}_i-\mathbf{x}_j\right\|^2}{\sigma^2}\right), &\mathbf{x}_i \in \mathcal{N}_k(\mathbf{x}_i) \cup \mathcal{N}_k(\mathbf{x}_j);\\
0,  & \mbox{otherwise},
\end{array}
\right.
\end{equation}
where $\mathcal{N}_k(\mathbf{x}_i)$ is the set of $k$-nearest neighbors of $\mathbf{x}_i$, $k$ and $\sigma$ is the pre-defined parameters. The local geometric structure of data is represented by this adjacency matrix $\mathbf{S}_0$ and the Laplace matrix $\mathbf{L}_0$ is obtained by
$\mathbf{L}_0=\mathbf{D}_0-\mathbf{S}_0$,
where $\mathbf{D}_0$ is a diagonal matrix with $(D_0)_{ii}=\sum_{j=1}^n (S_0)_{ij},~i=1,2,\cdot\cdot\cdot,n$.

To combine the local and global information of sample points together for learning more accurate pseudo labels, we propose to use a weighted form of two Laplace matrices
and consider the following objective function,
\begin{eqnarray}
\label{golfs0}
\min_{\mathbf{F,W}} ~ Tr(\mathbf{F}'(\mathbf{L}_1+\lambda\mathbf{L}_0)\mathbf{F}) +\alpha(\left\| \mathbf{XW}-\mathbf{F}\right\|_F^2+\beta\left\|\mathbf{W}\right\|_{2,1}),
\end{eqnarray}
where $\mathbf{F}=\mathbf{Y}(\mathbf{Y}'\mathbf{Y})^{-\frac{1}{2}}$ is defined as a scaled cluster indicator matrix, and $\lambda$ is the weighting parameter to combine $\mathbf{L}_1$ and $\mathbf{L}_0$. Furthermore, to solve the NP-hard problem \citep{SM2000} due to the discrete-valued elements of $\mathbf{F}$ and require the non-negativity of elements of $\mathbf{F}$, we follow the idea of the NDFS \cite{Ld2012u} to impose both the orthogonal constraint $\mathbf{F}'\mathbf{F}=\mathbf{I}_c$ and the nonnegative constraint $\mathbf{F}\geq 0$ in the optimization method. Therefore, we rewrite the objective function of (\ref{golfs0}) as
\begin{eqnarray}
\label{golfs}
\min_{\mathbf{F,W}} ~ Tr[\mathbf{F}'(\mathbf{L}_1+\lambda \mathbf{L}_0)\mathbf{F}] +\alpha(\left\|\mathbf{X}\mathbf{W}-\mathbf{F}\right\|_F^2+\beta\left\|\mathbf{W}\right\|_{2,1}), \\
\mbox{subject to} ~~ \mathbf{F}'\mathbf{F}=\mathbf{I}_c, \mathbf{F}\geq 0, \notag
\end{eqnarray}
where $\lambda,\alpha,\beta$ are the tuning parameters.
The first term of (\ref{golfs}) is the learning process of the pseudo labels based on both the local and global structure of data samples.
The second and third terms are the same as the NDFS method proposed in \cite{Ld2012u} for learning the feature selection matrix by a regression model between the pseudo labels and high-dimensional features with $l_{2,1}$-norm regularization. The $l_{2,1}$-norm penalty of $\mathbf{W}$ encourages the row sparsity of $\mathbf{W}$, which means that some rows of $\mathbf{W}$ can be estimated to be zeros or close to zero and the corresponding features can be deleted for spectral clustering.

Therefore, this new unsupervised feature selection method is
named GlObal and Local information combined Feature Selection (GOLFS) for high dimensional spectral clustering.
We remark that when only $\mathbf{L}_0$ is considered in the first term of (\ref{golfs}), the GOLFS becomes the NDFS. Thus, the GOLFS method can be considered a natural extension of the NDFS method by additionally incorporating the global correlation structure of data samples.

\subsection{An optimization algorithm}
The optimization procedure of the GOLFS method consists of two stages. The first stage is to learn the global correlation structure of data samples via solving the optimization problem (\ref{global}) and the second stage is to learn
the scaled cluster indicator matrix $\mathbf{F}$ and
the coefficient matrix $\mathbf{W}$ via solving the optimization problem (\ref{golfs}).

In the first stage, we applied the iterative reweighted least-squares (IRLS) algorithm proposed by \cite{Zd2015,Lange2000}
to solve the optimal $\widehat{\mathbf{P}}$ of (\ref{global}).
Given the current estimation $\mathbf{P}^{t}$,
we can update the estimate of $\boldsymbol{P}$ via
\begin{equation}
\boldsymbol{P}^{t+1}=\left[\left(\boldsymbol{G}_{2}^{t}\right)^{-1} \boldsymbol{X} \boldsymbol{G}_{1}^{t} \boldsymbol{X}'+\kappa \boldsymbol{I}_n\right]^{-1}\left[\left(\boldsymbol{G}_{2}^{t}\right)^{-1} \boldsymbol{X}  \boldsymbol{G}_{1}^{t} \boldsymbol{X}' \right],
\end{equation}
where $\boldsymbol{I}_n \in \mathbb{R}^{n\times n}$ is the identity  matrix,
$\boldsymbol{G}_{1}^{t}$ is a $d\times d$ diagonal weighting matrix with diagonal elements
$g_{1, j}^{t} = 1/\max(2\left\|\widetilde{\boldsymbol{x}}_{j}'-\widetilde{\boldsymbol{x}}_{j}' \boldsymbol{P}^{t}\right\|_2,\varsigma)$ for $j=1, 2, \ldots, d$,
$\widetilde{\boldsymbol{x}}_{j}\in \mathbb{R}^{n\times 1}$ is the $j$-th column vector of $\boldsymbol{X}$,
$\boldsymbol{G}_{2}^{t}$ is an $n\times n$ diagonal weighting matrix with diagonal elements
$g_{2, i}^{t}=1/\max(2\left\|\boldsymbol{p}_{i}^{t}\right\|_{2},\varsigma)$ for $i=1, 2, \ldots, n$,
 and $\varsigma$ is a sufficiently small positive value.
 The details of this step to solve the global structure matrix $\widehat{\mathbf{P}}$ are shown in Algorithm \ref{GOLFS}.

Once we have learned the global correlation of samples, we can modify the algorithm in \citep{Ld2012u}
to solve the objective function (\ref{golfs}) of the GOLFS. To be specific, we propose an iterative algorithm to solve $\mathbf{F}$ and $\mathbf{W}$. We first use a regularized term $\left\| \mathbf{F}'\mathbf{F}-\mathbf{I}_c\right\|_F^2$ to replace the orthogonal constraint of $\mathbf{F}$ with a tuning parameter $\gamma$  and rewrite the objective function (\ref{golfs}) of the GOLFS as
\begin{equation}
\label{Tansformed}
\mathcal{L}(\mathbf{F,W}) = Tr[\mathbf{F}'(\mathbf{L}_1+\lambda\mathbf{L}_0)\mathbf{F}]
+\alpha( \left\| \mathbf{X}\mathbf{W}-\mathbf{F}\right\|_F^2+\beta\left\|\mathbf{W}\right\|_{2,1}) + \frac{\gamma}{2}\left\|\mathbf{F}'\mathbf{F}-\mathbf{I}_c\right\|_F^2.
\end{equation}
So the objective function of GOLFS (\ref{golfs}) can be written as
\begin{align}
\label{GOLFS2}
\min_{\mathbf{F},\mathbf{W}} \mathcal{L}(\mathbf{F},\mathbf{W}), ~~~
s.t. ~ \mathbf{F}\geq 0.
\end{align}
The first order condition indicates that
\begin{equation}
\frac{\partial \mathcal{L}(\mathbf{F},\mathbf{W})}{\partial \mathbf{W}}=2\alpha [\mathbf{X}'(\mathbf{X}\mathbf{W}-\mathbf{F})+\beta\mathbf{D}\mathbf{W}]=0,
\end{equation}
where $\mathbf{D}$ is a diagonal matrix with $D_{ii}=\frac{1}{\left\|2\mathbf{w}_i\right\|_2}$.
Thus, the updated $\mathbf{W}$ can be finally written as
$
\mathbf{W}=(\mathbf{X}'\mathbf{X}+\beta\mathbf{D})^{-1}\mathbf{X}'\mathbf{F}.
$
Substituting this expression of $\mathbf{W}$ into the objective function (\ref{GOLFS2}), we can rewrite it as
\begin{align}
\label{214}
\min_{\mathbf{F}} Tr[\mathbf{F}'(\mathbf{L}_1+\lambda\mathbf{L}_0)\mathbf{F}] +Tr(\mathbf{F}'\mathbf{MF}) + \frac{\gamma}{2}\left\| \mathbf{F}'\mathbf{F}-\mathbf{I}_c\right\|_F^2,
s.t. ~\mathbf{F}\geq 0,
\end{align}
where $\mathbf{M}=\alpha[\mathbf{I}_n-\mathbf{X}(\mathbf{X'}\mathbf{X}+\beta\mathbf{D})^{-1}\mathbf{X}']$ and $\mathbf{I}_n\in  \mathbb{R}^{n\times n}$ is an identity matrix. Following \cite{Ld2012u},  we use the Lagrange method to eliminate the positive constrains and obtain the Lagrange function
\begin{align}
\phi(\mathbf{F},\mathbf{\Phi}) = Tr[\mathbf{F}'(\mathbf{L}_1+\lambda\mathbf{L}_0)\mathbf{F}]  +Tr(\mathbf{F}'\mathbf{MF}) + \frac{\gamma}{2}\left\| \mathbf{F}'\mathbf{F}-\mathbf{I}_c\right\|_F^2 + Tr(\bm{\Phi}\mathbf{F}'),
\end{align}
where $\bm{\Phi}$ is the matrix of Lagrange multipliers. Similar to \cite{Ld2012u}, by solving the first order condition and using the Karush-Kuhn-Tucker condition \citep{KT1951}, we can generate the update of $\mathbf{F}$ by
\begin{equation}
\label{ }
\mathbf{F}_{ij} \leftarrow \frac{(\gamma\mathbf{F})_{ij}}{(\mathbf{L}_1\mathbf{F}+\lambda\mathbf{L}_0\mathbf{F}+\mathbf{MF}+\gamma\mathbf{F}\mathbf{F}'\mathbf{F})_{ij}} \mathbf{F}_{ij}.
\end{equation}

\begin{algorithm}
\caption{GlObal and Local info. combined Feature Selection (GOLFS)}
\label{GOLFS}
\textbf{Input}:  $\mathbf{X} \in \mathbb{R}^{d\times n}$, $\lambda, \alpha, \beta, \gamma, \kappa$ \\
\textbf{Stage 1}: Learning the global correlation structure of samples
\begin{algorithmic}[1]
	\State set $t=1$; initialize $\boldsymbol{G}_{1}^{0}$  and $\boldsymbol{G}_{2}^{0}$ \\
\textbf{repeat}
	\State ~~~~~~ $\boldsymbol{P}^{t+1}=\left[\left(\boldsymbol{G}_{2}^{t}\right)^{-1} \boldsymbol{X} \boldsymbol{G}_{1}^{t} \boldsymbol{X}'+\kappa \boldsymbol{I}_n\right]^{-1}\left[\left(\boldsymbol{G}_{2}^{t}\right)^{-1} \boldsymbol{X}  \boldsymbol{G}_{1}^{t} \boldsymbol{X}' \right]$
	\State ~~~~~~ update $\boldsymbol{G}_{1}^{t+1}$  and $\boldsymbol{G}_{2}^{t+1}$ based on $\boldsymbol{P}^{t+1}$
	\State ~~~~~~ $t=t+1$
	\State \textbf{until} the convergence criterion satisfied
	\State construct the symmetric global similarity matrix $\mathbf{S_1}$ and the corresponding Laplace matrix $\mathbf{L_1}$
\end{algorithmic}

\textbf{Stage 2}:  Learn pseudo labels and select important features
\begin{algorithmic}[1]
    \State  initialize Laplace matrix $ \mathbf{L}_0$ based on k-adjacent matrix
    \State set $t=1$; initialize $\mathbf{F}^t$, $\mathbf{W}^t$, $\mathbf{D}^t$ \\
\textbf{repeat}
	\State ~~~~~~ $\mathbf{M}^t=\alpha[\mathbf{I}_n-\mathbf{X}(\mathbf{X'}\mathbf{X}+\beta\mathbf{D}^t)^{-1}\mathbf{X}']$
	\State ~~~~~~ $\mathbf{F}_{ij}^{t+1} = \frac{(\gamma\mathbf{F}^t)_{ij}}{(\mathbf{L_1F}^t+\lambda\mathbf{L_0}\mathbf{F}^t+\mathbf{MF}^t+\gamma \mathbf{F}^t (\mathbf{F}^t)'   \mathbf{F}^t)_{ij}} \mathbf{F}^t_{ij}$
	\State ~~~~~~ $\mathbf{W}^{t+1} = (\mathbf{X}\mathbf{X}'+\beta\mathbf{D}')^{-1}\mathbf{X}\mathbf{F}^{t+1}$
	\State ~~~~~~ update the diagonal matrix $\mathbf{D}^{t+1}$ with $D^{t+1}_{ii}=\frac{1}{2 \left\| \mathbf{w}_{i}^{t+1}\right\|_2}$
	\State ~~~~~~ $t=t+1$
	\State \textbf{until} the convergence criterion satisfied
\end{algorithmic}
\textbf{Output}: $\mathbf{W}$ and ordered features based on $\left\|\mathbf{w}_i\right\|_2$
\end{algorithm}

Algorithm 1 in the following summarizes the details of the GOLFS procedure. In the following Proposition 1, we demonstrate the convergence of the algorithm by showing that the transformed objective function $\mathcal{L}(\mathbf{F}, \mathbf{W})$ monotonically decreases in the second stage of the iterative algorithm. The proof can be found in the Appendix. It guarantees that we can obtain the local optimal $\mathbf{W}$ by Algorithm \ref{GOLFS}. Experiments based on simulations and real-world data sets support that the algorithm converges quickly.
\begin{proposition}
\label{CT}
The value of the transformed objective function $\mathcal{L}(\mathbf{F, W})$ of the GOLFS algorithm defined in (\ref{Tansformed}) monotonically decreases in each iteration with the updating rules proposed in Algorithm \ref{GOLFS}. That is, for $t$-th iteration,
$$
\mathcal{L}(\mathbf{F}^{t+1}, \mathbf{W}^{t+1}) \leq \mathcal{L}(\mathbf{F}^{t+1},\mathbf{W}^t) \leq \mathcal{L}(\mathbf{F}^{t},\mathbf{W}^t).
$$
\end{proposition}

\section{Simulation}
This section presents the simulation results for GOLFS with comparisons to four classical embedded methods. The finite sample performance of feature selection and clustering are evaluated in two different data-generating processes.

\subsection{Data generating process}
 We consider two datasets generated with some common settings: the sample size $n=200$, the number of true clusters $k=5$, and the number of features $p=1000$ containing $q=10$ true features. For each simulation example, true features and irrelevant features are generated independently. In each simulated example, we generate $40$ observations for each cluster and repeat the whole process 100 times.  
 
\vspace{0.3cm}

\noindent\textbf{Example 1}: All features are identically and independently distributed with normal distribution. We denote $X_{(k)}^q$ as the observed value of the $q$-th true feature and the $k$-th cluster. The true feature vectors are generated from the following normal distribution:
\begin{align}
\label{}
    X_{(k)}^q &\sim N(\mu_k, \sigma_q^2), ~ q=1,2,\cdot\cdot\cdot,10; ~ k=1,2,\cdot\cdot\cdot,5, \notag \\
    \mu_k &\sim U(1,10),~ k=1,2,\cdot\cdot\cdot,5, \\
    \sigma_k &\sim N(0,1),~q=1,2,\cdot\cdot\cdot,10. \notag
\end{align}
That is, the first 10 features of the observations in the same cluster are generated from a common mean and different variance. Then we merge the 200 observations generated (40  observations for each cluster) into the true feature matrix $\mathbf{X}_q^1\in \mathbb{R}^{200 \times 10}$ by row. The irrelevant features are generated by the following process:
\begin{align}
\label{}
    X^p &\sim N(\mu_p, \sigma_p^2), p=11,12,\cdot\cdot\cdot,1000, \notag  \\
    \mu_p &\sim U(1,10),~ p=11,12,\cdot\cdot\cdot,1000, \\
    \sigma_p &\sim N(0,1),~p=11,12,\cdot\cdot\cdot,1000. \notag
\end{align}
That is, the last 990 irrelevant features of the observations contain no cluster information. If we denote the irrelevant feature matrix as $\mathbf{X}_p^1\in \mathbb{R}^{200 \times 990}$, then the dataset of Example 1 is $\mathbf{X}^1 \in \mathbb{R}^{200 \times 1000}$ by merging $\mathbf{X}_q^1$ and $\mathbf{X}_p^1$ by column.

\vspace{0.3cm}
\noindent\textbf{Example 2}: All features are normally distributed with AR(1) correlations. $\mathbf{X}_{(k)}$ denotes the observed values for the $k$-th cluster and $(\bm{\mu}_{k})_q$ represent the $q$-th element of the mean vector for the $k$-th cluster. The true features are generated from the following multivariate normal distribution:
\begin{align}
\label{}
    \mathbf{X}_{(k)} & \sim N(\bm{\mu}_k, \bm{\Sigma}) ~ k=1,2,\cdot\cdot\cdot,5, \notag \\
    (\bm{\mu}_{k})_q & \sim U(1,10),~ k=1,2,\cdot\cdot\cdot,5 ; ~ q=1,2,\cdot\cdot\cdot,10, \\
    \Sigma_{ij} &  = 0.5^{\lvert i-j\lvert},~i,j=1,2,\cdot\cdot\cdot,10.\notag
\end{align}
That is, the first 10 features of the observations are generated from a multi-normal distribution with a common covariance structure and different mean vectors for different clusters. Then we merge the data generated for different clusters by rows into the true feature matrix $\mathbf{X}_q^2\in\mathbb{R}^{200 \times 10}$. The irrelevant features are generated by the following process:
\begin{align}
\label{}
    \mathbf{X}^p & \sim N(\bm{\mu}^p, \bm{\Sigma}^p), \notag \\
    (\bm{\mu}^p)_q & \sim U(1,10), ~ q=11,12,\cdot\cdot\cdot,1000, \\
    \bm{\Sigma}^p_{ij} & = 0.5^{\lvert i-j\lvert}, ~ i,j=11,12,\cdot\cdot\cdot,1000, \notag
\end{align}
where $(\bm{\mu}^p)_q$ is the $q$-th element of vector $\bm{\mu}^p$. The simulated irrelevant features is denoted as $\mathbf{X}_p^2\in \mathbb{R}^{ 200\times 990}$. The dataset of {Example 2} is $\mathbf{X}^2 \in \mathbb{R}^{ 200\times 1000}$ generated by merging $\mathbf{X}_q^2$ and $\mathbf{X}_p^2$ by columns.

\subsection{Evaluation criteria}
As we know the subset of true features in simulations, the number of true features selected (True Positive, TP), and the probability to cover all true features (Coverage Probability, CP) within the top $s\in\{10,30,60\}$ features selected are used to evaluate the feature selection performance of the UFS methods. For the simulation examples with $m$ repeats, the criteria can be computed as
\begin{align}
    TP(s) &= \frac{1}{m} \sum_{k=1}^m \sum_{j=1}^{s} I\left(T_{j(k)}\subset T\right),\\
    CP(s) &= \frac{1}{m} \sum_{k=1}^m I\left(\bigcup_{j=1}^{s} T_{j(k)}\supset T\right),
\end{align}
where $T$ denotes the set of true features, $T_{j(k)}$ denotes the $j$-th top feature selected in the $k$-th repeat and $I(\cdot)$ is an indicator function.

Three commonly used evaluation metrics of clustering results are clustering accuracy (ACC), normalized mutual information (NMI), and adjusted Rand index (ARI). Larger values of ACC, NMI, and ARI indicate better performance for clustering. The clustering accuracy can be computed as
\begin{equation}
\label{ }
ACC=\frac{1}{n}\sum_{i=1}^n \delta(c_i, map(c_i')),
\end{equation}
where $n$ is the sample size, $c_i$ is the true label of $i$-th sample and $map(\cdot)$ is a best mapping from clustering label $c_i'$ to true labels. $\delta(x,y)=1$ when $x=y$ and $\delta(x,y)=0$ otherwise.

Normalized mutual information can be computed as
\begin{align}
\label{}
    NMI(C,C')= \frac{MI(C,C')}{H(C)H(C')},
\end{align}
where $C$ and $C'$ represent the sets of the true labels and the clustering labels, respectively, the numerator is the mutual information of $C$ and $C'$ $MI(C,C') = \sum_{c_i\in C, c_j\in C'} P(c_i,c_j)\log_2[P_{C,C'}(c_i,c_j)/P_C(c_i)P_{C'}(c_j)]$, $P_{C}(\cdot)$ and $P_{C'}(\cdot)$ are the marginal distribution of true labels and clustering labels respectively and $P_{C,C'}(\cdot)$ is the corresponding joint distribution, $H(C)$ and $H(C')$ are the entropies of $C$ and $C'$, respectively. $NMI(C,C')$ belongs to $[0,1]$ and measures how much the two sets of labels coincide with each other.

\subsection{Simulation results}

To validate the finite sample performance of GOLFS for both feature selection and clustering, we compare it with the four classical embedded UFS methods: MCFS \citep{Cd2010}, UDFS \citep{Yd2011}, NDFS \citep{Ld2012u} and RSR \citep{Zd2015}. Besides, we also use the K-Means method
using all features without feature selection as a benchmark method. We set $k=5$ for all methods to construct the neighbor graph and use the grid search method to choose tuning parameters with optimal clustering performance. The simulation results for finite performances of feature selection are summarized in Table \ref{Simulation results}. Firstly, the GOLFS has the dominant performance in feature selection in terms of TP and CP. As the results show, the average number of true features selected by the GOLFS is significantly larger and the probabilities to cover all true features are higher than the other UFS methods. Secondly, when data contains some correlation structure, the GOLFS has robust performances in feature selection compared to the NDFS due to the additional consideration of the global correlation structure of data samples.

\begin{table}[h]
\caption{Performances of feature selection. TP denotes the true positives and CP denotes the coverage probability of all true features in the top $s\in\{10,30,60\}$ selected features.
}
 \label{Simulation results}
 \centering
\begin{tabular}{@{}crrrrrrr@{}}
\toprule
 & \textbf{Methods} & \textbf{TP(10)} & \textbf{TP(30)} & \textbf{TP(60)} & \textbf{CP(10)} & \textbf{CP(30)} & \textbf{CP(60)} \\ \midrule
\multirow{5}{*}{Example 1}            & MCFS             & 3.98          & 4.94          & 5.64          & 0.01          & 0.02          & 0.04          \\
                                         & UDFS             & 2.4           & 3.52          & 4.24          & 0             & 0             & 0             \\
                                         & NDFS             & 2.23          & 3.22          & 5.42          & 0             & 0             & 0.01          \\
                                         & RSR              & 0             & 0             & 0             & 0             & 0             & 0             \\
                                         & \textbf{GOLFS}             & \textbf{8.81} & \textbf{9.29} & \textbf{9.48} & \textbf{0.77} & \textbf{0.88} & \textbf{0.9}  \\ \midrule
\multirow{5}{*}{Example 2}            & MCFS             & 3.38          & 4.43          & 5.04          & 0             & 0.01          & 0.01          \\
                                         & UDFS             & 2.35          & 3.31          & 4.03          & 0             & 0             & 0             \\
                                         & NDFS             & 1.25          & 2.22          & 3.85          & 0             & 0             & 0             \\
                                         & RSR              & 0             & 0             & 0.01          & 0             & 0             & 0             \\
                                         & \textbf{GOLFS}             & \textbf{6.18} & \textbf{7.64} & \textbf{8.25} & \textbf{0.36} & \textbf{0.61} & \textbf{0.72} \\ \bottomrule
\end{tabular}
\end{table}

In practice, the subset of true features is often not available. The accuracy of clustering based on the selected features is commonly used to evaluate the performance of feature selection. We use ACC and NMI of the $K$-means clustering results as evaluations for different UFS methods. The dimension of the top selected features is set to be $\{10,20,\cdot\cdot\cdot,150\}$, and the averages of evaluation criteria are computed. Similar to the settings in \cite{Yd2011} and \cite{Ld2012}, we tune the parameters for each feature selection algorithm from $\{10^{-6}$, $10^{-5}$, $\cdot\cdot\cdot$, 1 ,$\cdot\cdot\cdot$,$10^{5}$,$10^{6}\}$ and record the optimal performances. The evaluation processes are repeated for $20$ times in the $K$-means clustering. The simulation results for performance in clustering are summarized in Table \ref{Simulation results2}. As the results show, the GOLFS has better finite sample performances with high accuracy in clustering compared to existing embedded UFS methods.
It demonstrates that it is of great importance to consider the global correlations among data samples to improve the performance of the embedded UFS method. It is also interesting to observe that the RSR which is only based on the correlation structure among features can not perform well in either feature selection or clustering. In conclusion, the combination of the
local geometric structure via manifold learning and the global correlation structure of data samples via regularized self-representation in the GOLFS algorithm is able to improve the accuracy of both feature selection and clustering.

\begin{table}[h]
\caption{The comparison of embedded UFS methods in performances of clustering. ACC denotes the clustering accuracy, NMI denotes the normalized mutual information and ARI denotes the adjusted Rand index. The mean evaluations are reported for different methods and the corresponding standard deviations are presented in parentheses.}
\label{Simulation results2}
\centering
\small
\centering
 
\begin{tabular}{@{}ccccc@{}}
\toprule
\textbf{Simulations}       & \textbf{Method} & \textbf{NMI }     & \textbf{ACC ($\%$)}      & \textbf{ARI }     \\ \midrule
\multirow{6}{*}{Example 1} & Benchmark         & 0.105(0.019)             & 31.62 (2.79)           & 0.033 (0.012)            \\
                           & MCFS                                  & 0.057 (0.008)           & 30.37 (2.09)           & 0.003 (0.007)           \\
                           & NDFS                                  & 0.378 (0.022)            & 45.32 (2.73)           & 0.174 (0.018)          \\
                           & RSR                                   & 0.187 (0.021)             & 38.00 (2.55)            & 0.082 (0.018)           \\
                           & UDFS                                  & 0.064 (0.013)            & 31.08 (0.58)           & 0.007 (0.011)           \\
                           & \textbf{GOLFS}  & \textbf{0.384} (0.005) & \textbf{45.92 }(0.63)  & \textbf{0.179} (0.004) \\ \midrule
\multirow{6}{*}{Example 2} & Benchmark       & 0.307 (0.047)          & 43.00 (3.65)           & 0.144 (0.048)          \\
                           & MCFS             & 0.192 (0.060)          & 37.47 (3.61)           & 0.054 (0.048)           \\
                           & NDFS             & 0.517 (0.067)          & 53.67 (6.64)           & 0.304 (0.079)          \\
                           & RSR                & 0.495 (0.070)          & 55.32 (7.23)           & 0.300 (0.083)          \\
                           & UDFS              & 0.207 (0.061)          & 38.91 (5.08)           & 0.065 (0.069)           \\
                           & \textbf{GOLFS}  & \textbf{0.525} (0.058) & \textbf{55.13} (3.70) & \textbf{0.315} (0.049) \\ \bottomrule  
\end{tabular} 
 
\end{table}
 
\section{Real Data Analysis}

In this section, we apply our proposed method GOLFS to two real datasets to test the empirical performance in the applications of text clustering and image clustering.

\vspace{0.3cm}
\noindent\textbf{Example 3} (Text clustering): The BBC news dataset \footnote{Download this public dataset from https://storage.googleapis.com/laurencemoroney-blog.appspot.com/bbc-text.csv}  contains 200 reports of news. Each one of the reports contains $400$ words on average and belongs to one of the five topics, e.g. economy, politics, sports, technology, and entertainment. For this dataset with text sequences, we segment the words and use the evaluation ``term frequency-inverse document frequency" (TF-IDF) \citep{R2003} to transform the text sequences into a structured dataset with the features in columns and the observations in rows. TF-IDF is a popular measure in text mining that evaluates every word by its frequency in a document and its uniqueness among various documents. Thus, TF-IDF can efficiently identify the common words, stop words, and keywords of an article. After this process, we finally obtain the data matrix with the sample size $n=200$ and the dimension of features $p=2057$. Then, the GOLFS is used to select the true discriminative keywords and cluster the articles using K-means clustering based on the selected words.

\begin{figure}[H]
\begin{center}
  \subfigure[cluster 1]{
    \includegraphics[height=1.7in]{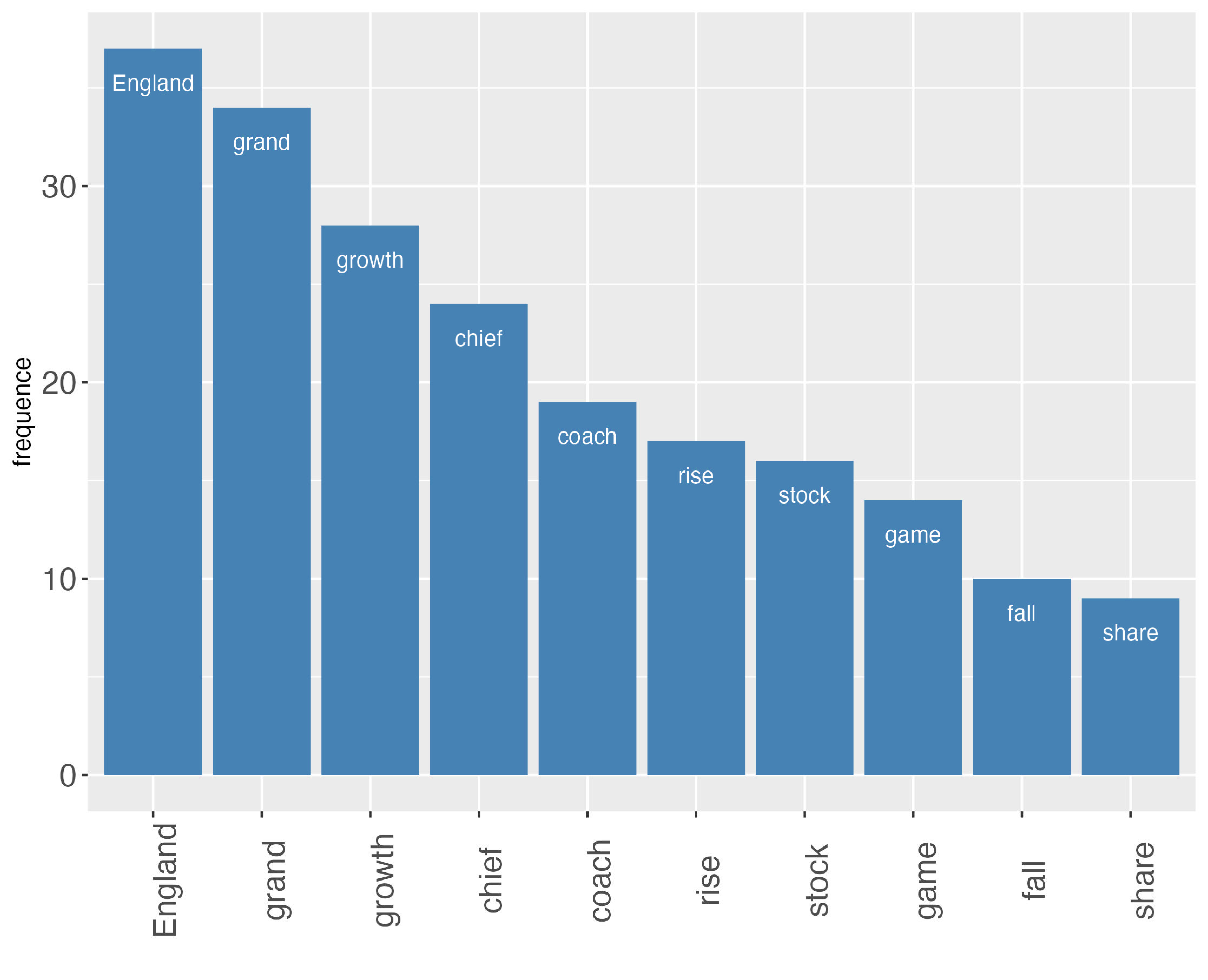}
  }
  \subfigure[cluster 2]{
    \includegraphics[height=1.7in]{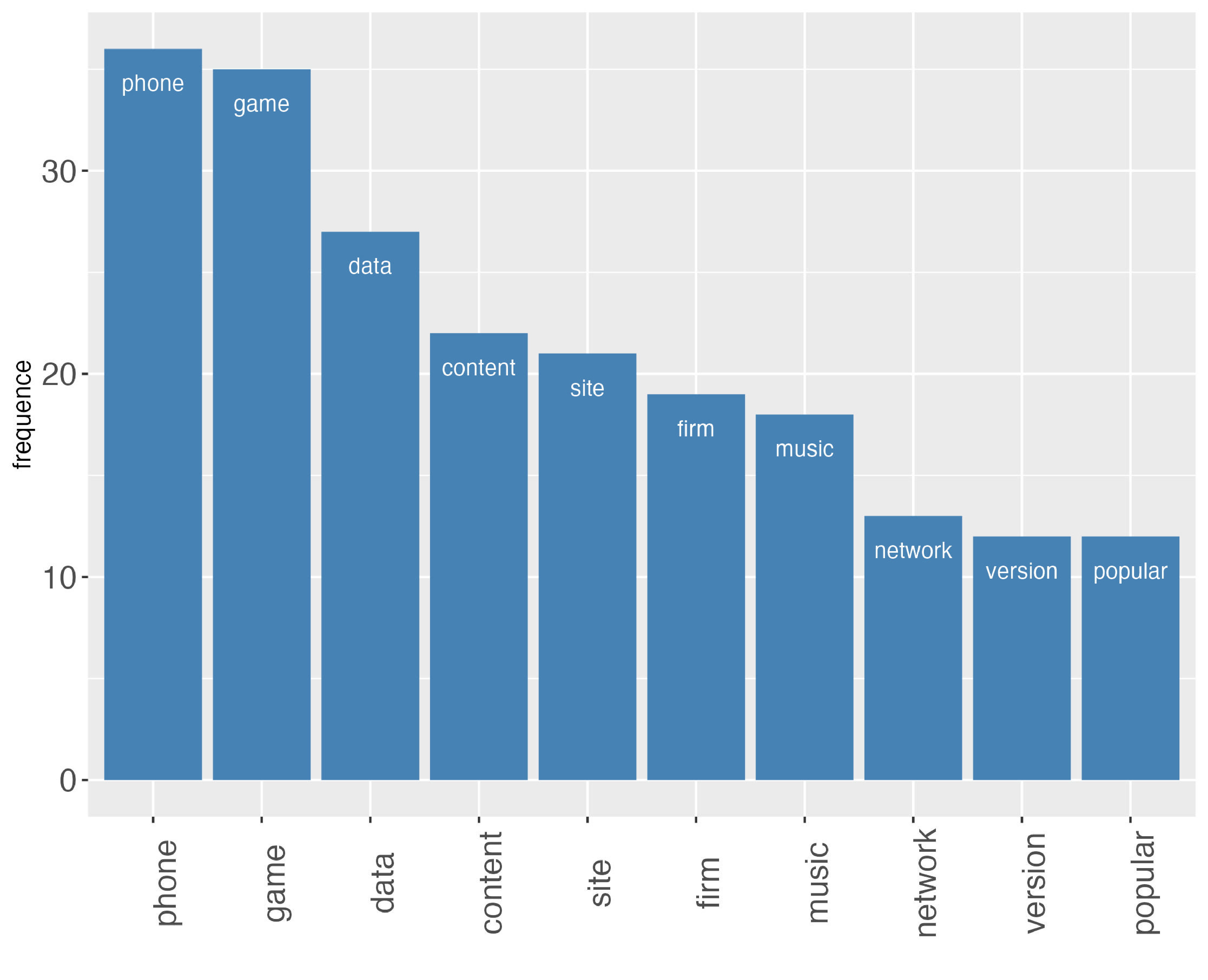}
  }
  \subfigure[cluster 3]{
    \includegraphics[height=1.7in]{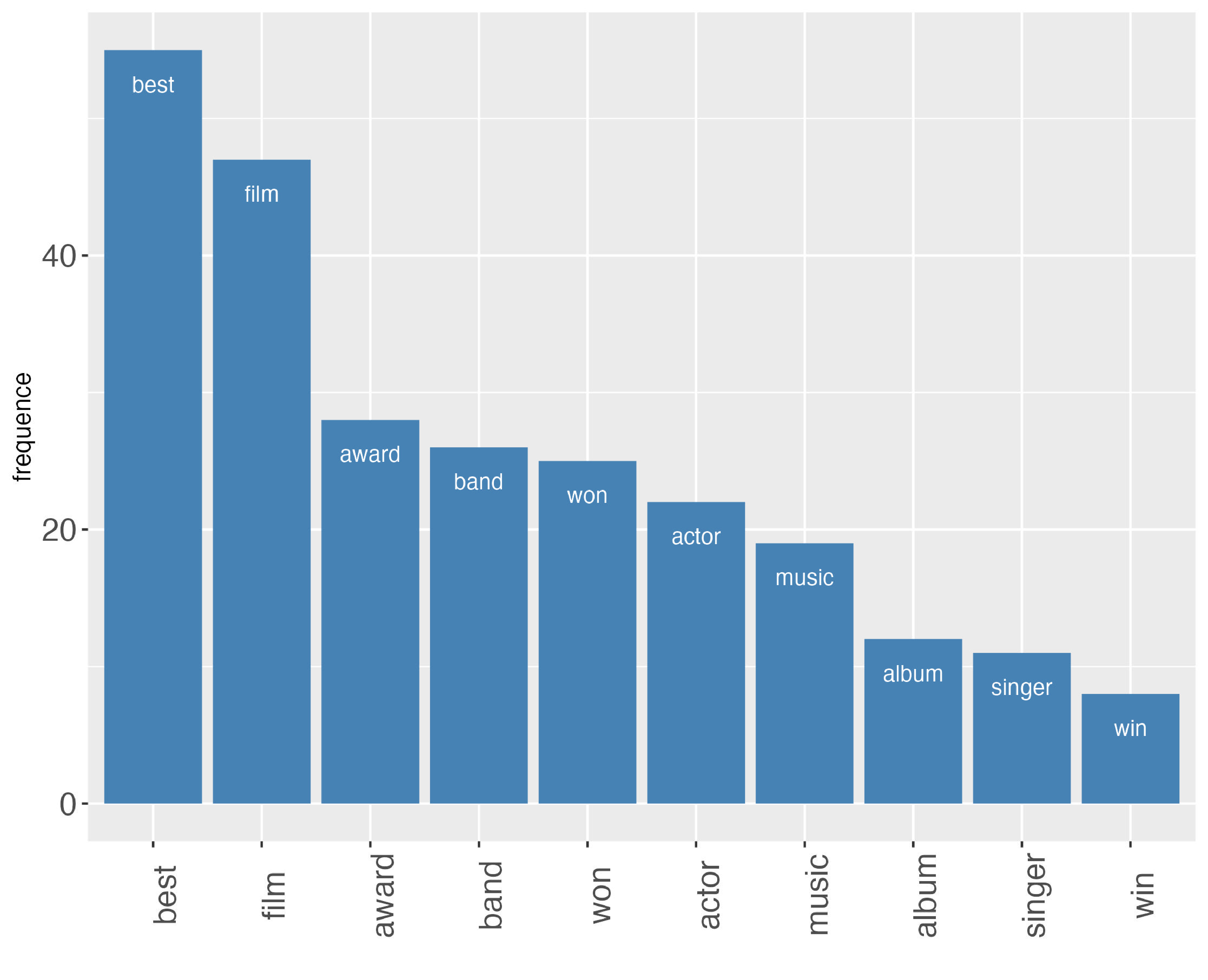}
  }
  \subfigure[cluster 4]{
    \includegraphics[height=1.7in]{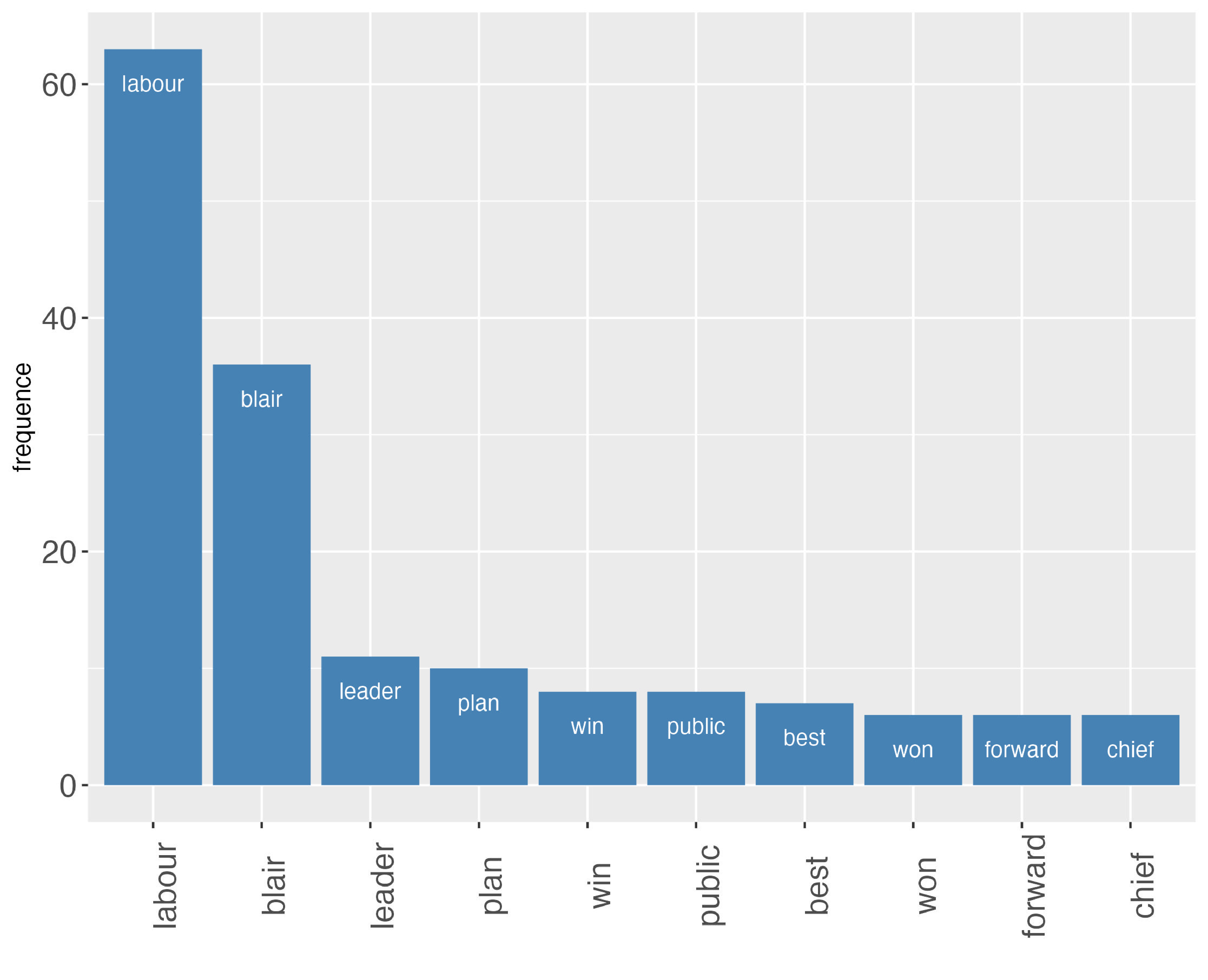}
  }
  \subfigure[cluster 5]{
    \includegraphics[height=1.7in]{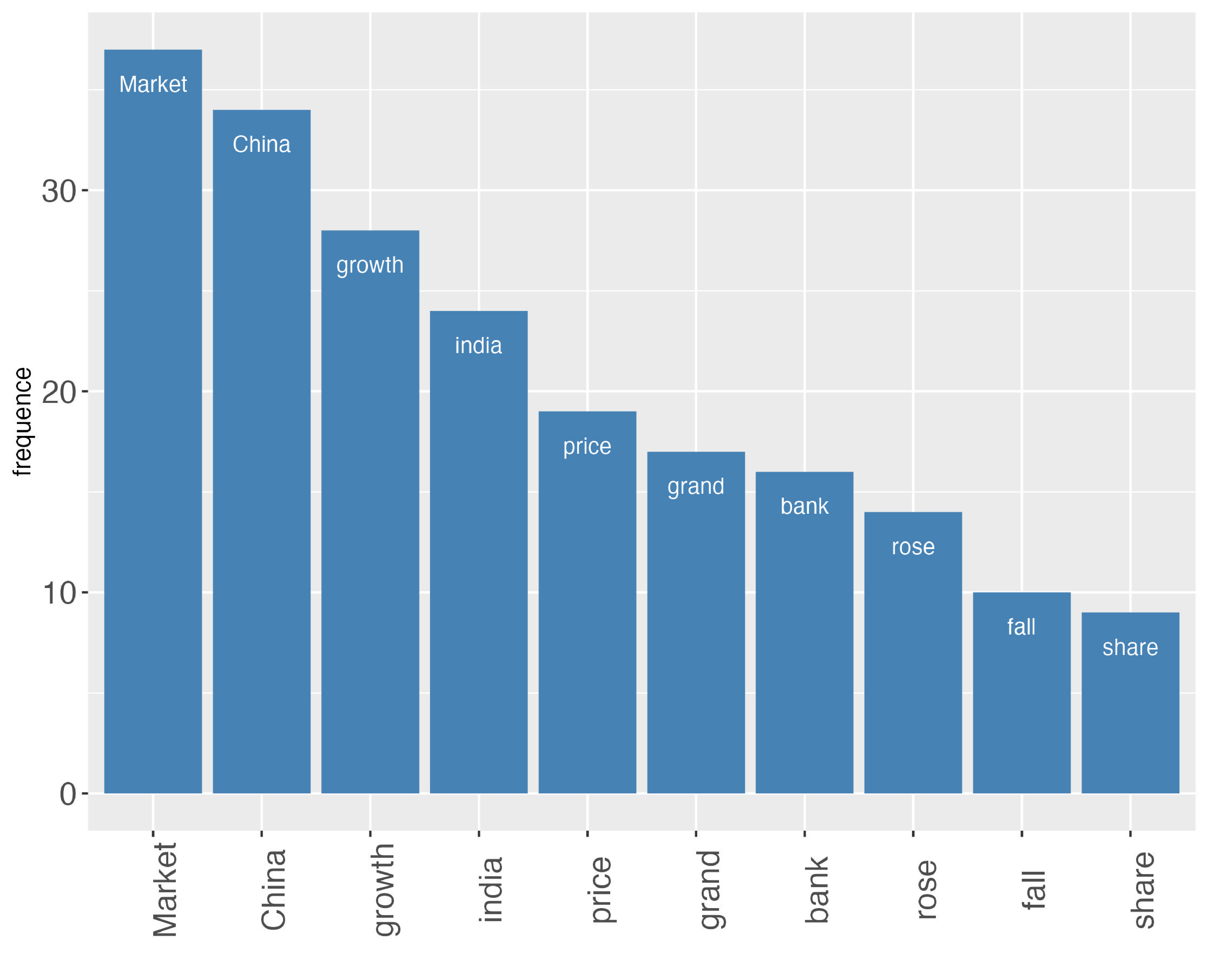}
  }
\end{center}
\caption{Word frequency of the top 10 keywords selected by the GOLFS for 5 clusters.}
\label{BBCfig}
\end{figure}

The important features selected by the GOLFS are shown as word clouds in Figure \ref{BBCfig}  with different frequency for 5 topics and the confusion matrix of clustering are shown in Table \ref{cfm}. We can find the keywords selected are discriminative and reasonable and the clustering results are good. The comparison to classic embedded methods is shown in Table \ref{Performances of clustering in real datasets}, and the results show the better performance of the GOLFS with the highest ACC, NMI and ARI.

\begin{table}[h]
\caption{Confusion matrix of clustering based on the selected features via the GOLFS}
\label{cfm}
\begin{center}
\begin{tabular}{@{}c|ccccc@{}}
\toprule
         & Sports  & Technology & Entertainment & Policy & Economy \\ \midrule
Cluster 1 & 18     & 0          & 0             & 0      & 0       \\
Cluster 2 & 0      & 24         & 0             & 0      & 0       \\
Cluster 3 & 0      & 0          & 19            & 0      & 0       \\
Cluster 4 & 0      & 0          & 0             & 19     & 0       \\
Cluster 5 & 20     & 15         & 20            & 22     & 43      \\ \bottomrule
\end{tabular}
\end{center}
\end{table}
 
\vspace{0.3cm}
\noindent\textbf{Example 4} (Image clustering): The Yale32 dataset \footnote{Download this public dataset from http://www.cad.zju.edu.cn/home/dengcai/Data/FaceData.html} contains $165$ face images showing different facial expressions or configurations of $15$ people. Each one of the raw face images shown in Figure \ref{Yale32} is stored as a matrix with $32\times 32=1024$ pixels. For image data, we treat every single pixel as a feature with continuous gray scales from $0$ to $255$ and transform the image matrix into a long vector. Finally, we transform the face image dataset into a data matrix with $n=165$ and $p=1024$. We apply the GOLFS algorithm to this transformed dataset to select the discriminative set of pixels and do the clustering for face images.

Figure \ref{recoveryOFYale} shows the reconstruction using top $300$, $600$, and $900$ important features selected by the GOLFS. When we reconstruct the images with only the top $300$ features, the selected pixels can almost cover the area of the eyes, nose and mouth although the whole image seems blurred. When we use the top $600$ features to reconstruct the images, the contours of a face are almost clear. This result shows that the GOLFS can effectively select the discriminative features in the image data.
According to the results summarized in Table \ref{Performances of clustering in real datasets}, the GOLFS has the dominant performance in clustering with the highest ACC, NMI and ARI.

\begin{figure}[H]
\begin{center}
\includegraphics[height=1.2in]{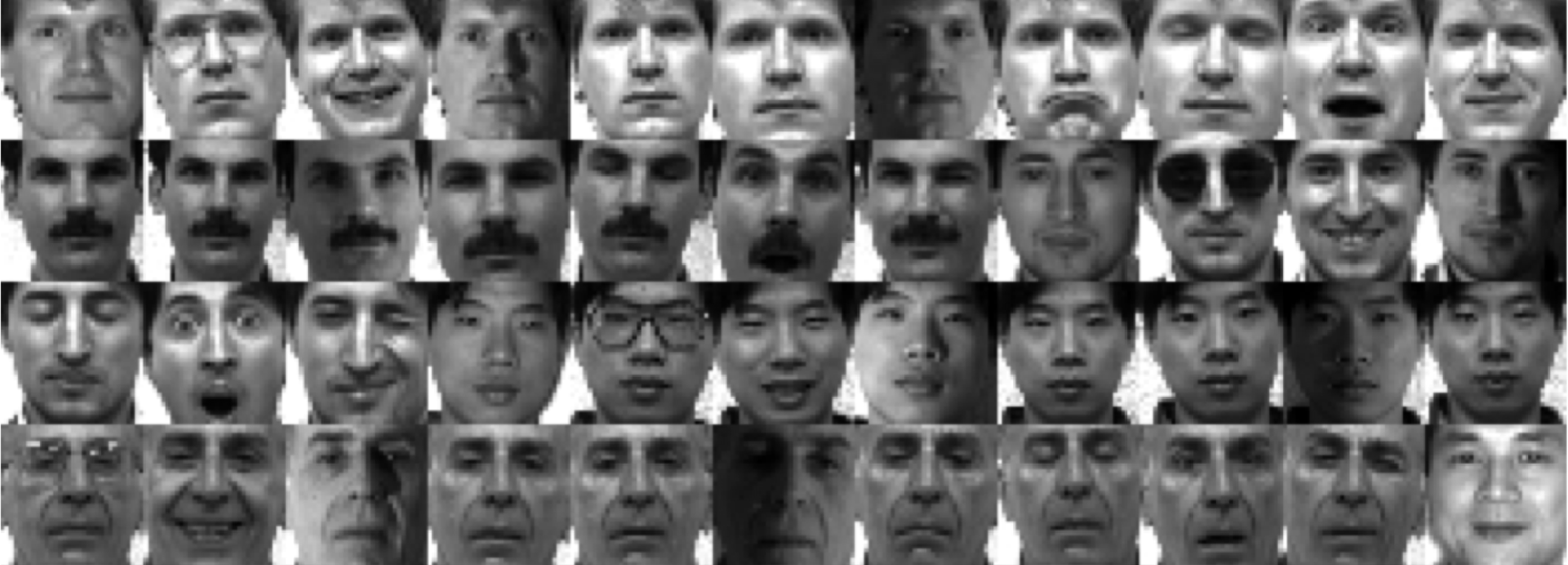}
		\caption{A sample of Yale32 image dataset. Each raw face image is stored as a $32\times 32$ matrix.}
		\label{Yale32}
\end{center}
\end{figure}

\begin{figure}[H]
\begin{center}
		\centering
		\includegraphics[height=1.2in]{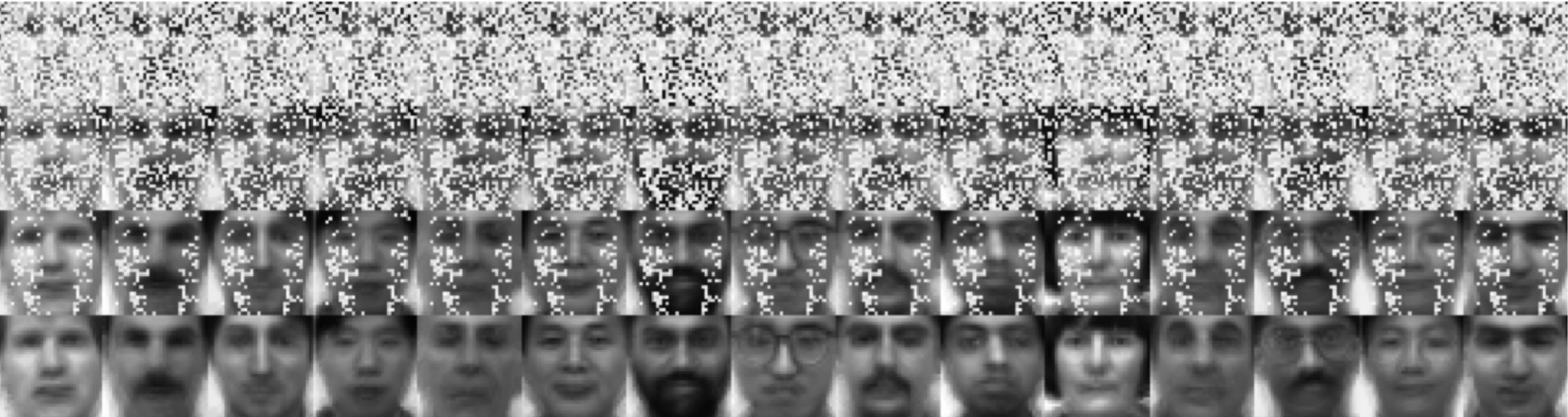}
		\caption{The reconstructed images using the top 300, 600 and 900 top features selected by the GOLFS. The last row denotes the original images of 15 people in the Yale32 dataset.}
		\label{recoveryOFYale}
\end{center}
\end{figure}

The comparisons based on two real datasets show the excellent finite sample performance of the GOLFS in clustering.
We tune the parameters via grid searching with the same settings in simulation and the best performances of methods from the optimal parameters are reported. 
For these two real datasets, the GOLFS and the NDFS perform much better than the other methods since the GOLFS inherits the advantage of the NDFS to capture the local geometric structure of samples. On the other hand, the GOLFS also performs better than the NDFS by additionally considering the global correlations of samples in the algorithm. The empirical results of the real data analysis support the advantage of the combination of both global and local information in the embedded UFS method.

\begin{table}[h]
\centering
\caption{The comparison of embedded UFS methods in real datasets. ACC denotes the clustering accuracy, NMI denotes the normalized mutual information and ARI denotes the adjusted Rand index.}
\label{Performances of clustering in real datasets}
\begin{tabular}{@{}ccccc@{}}
\toprule
\textbf{Dataset}      & \textbf{Methods} & \textbf{NMI}   & \textbf{ACC}   & \textbf{ARI}   \\ \midrule
\multirow{6}{*}{BBC}  & Benchmark        & 0.186          & 0.370          & 0.056          \\
                      & MCFS             & 0.233          & 0.365          & 0.054          \\
                      & NDFS             & 0.400          & 0.545          & 0.150          \\
                      & RSR              & 0.168          & 0.317          & 0.019          \\
                      & UDFS             & 0.233          & 0.358          & 0.061          \\
                      & \textbf{GOLFS}   & \textbf{0.427} & \textbf{0.580} & \textbf{0.176} \\
\midrule
\multirow{6}{*}{Yale} & Benchmark        & 0.177          & 0.042          & 0.045          \\
                      & MCFS             & 0.187          & 0.048          & 0.050          \\
                      & NDFS             & 0.220          & 0.048          & 0.057          \\
                      & RSR              & 0.236          & 0.056          & 0.067          \\
                      & UDFS             & 0.193          & 0.055          & 0.054          \\
                      & \textbf{GOLFS}   & \textbf{0.259} & \textbf{0.059} & \textbf{0.078} \\
\bottomrule
\end{tabular}
\end{table}

\section{Conclusion and Discussion}

In this paper, we propose a new embedded method for unsupervised feature selection named the GOLFS, which jointly considers both the local geometric structure via manifold learning and the global correlation structure of data samples via the regularized self-representation technique. With the combination of comprehensive information, the GOLFS is able to select the discriminative and representative features with high accuracy and achieve robust clustering results. Simulations and real data analysis have shown that the GOLFS can perform better than several existing embedded UFS methods. Although the performance of GOLFS is impressive, the tuning process may be relatively computational-costly for training. As the searching space of parameters is large and grid search for optimal parameters is not computationally efficient in many real applications, a quick searching strategy for tuning parameters in unsupervised learning could be an important topic.  Another future topic is developing statistical theorems for inference using unsupervised feature selection methods.

\newpage

\begin{appendices}

\section{Spectral clustering}

Spectral clustering is a method of clustering based on a graph. For an unlabeled dataset $\mathbf{X} \in \mathbb{R}^{d\times n}$ with $d$ features and $n$ observations, we usually define the distance between observation $\mathbf{x}_i$ and $\mathbf{x}_j$ as $s_{ij}$, which is usually the Euclidean distance. To show the local geometric structure of data, a graph based on the distances can be represented as an adjacency matrix $\mathbf{A} \in \mathbb{R}^{n\times n}$. The connections between observations are usually built by $k$-nearest neighbor method and the weighs $A_{ij}$ are given by $A_{ij}=\Phi(s_{ij})$, where $A_{ij}$ is the element of matrix $\mathbf{A}$ and $\Phi(\cdot)$ is usually Gaussian kernel function.

To find $k$ clusters is equivalent to cutting this graph into $k$ subgraphs. Intuitively, the loss of a cut between subsets $R_i$ and $R_j$ can be defined as $W(R_i, R_j) = \sum_{m\in R_i, n\in R_j} A_{mn}$. To make the cut of subgraph more robust, we defined the cut loss of $k$ subgraphs as weighted sum of every single cut loss $l(R_1,R_2,\cdot\cdot\cdot,R_k)=\frac{1}{2}\sum_{i=1}^k \frac{W(R_i,\bar{R}_i)}{\lvert R_i\lvert}$, where $\bar{R}_i$ represent the subgraph outside of $R_i$ and $ \lvert R_i\lvert$ is the number of points in the subgraph $R_i$. Then the optimization problem to minimize the total cut loss have an equivalent form as shown in Lemma 2.1.  Next, we provide the proof of Lemma 2.1.

We first introduce a lemma before we prove Lemma 2.1.
\begin{lemma}
\label{LM1}
For a symmetric matrix $\mathbf{A}$, define the row-summation matrix $\mathbf{D}=\left(\begin{array}{ccc}d_1 &  &  \\ & d_2 &  \\ &  &\ddots \end{array}\right)$, $d_i=\sum_{j=1}^n A_{ij}$, the corresponding Laplace matrix is $\mathbf{L}=\mathbf{D}-\mathbf{A}$. Laplace matrix is symmetric and semi-positive definite, and for any vector $\mathbf{f}$ :
$$
\mathbf{f}'\mathbf{Lf} = \frac{1}{2}\sum\limits_{i,j=1}^{n}A_{ij}(f_i-f_j)^2
$$
\end{lemma}

\begin{proof}[Proof of Lemma \ref{LM1}]
\begin{align*}
 \mathbf{f}'\mathbf{Lf} &= \mathbf{f}'\mathbf{Df} - \mathbf{f}'\mathbf{Af} = \sum\limits_{i=1}^{n}d_i(f_i)^2 - \sum\limits_{i,j=1}^{n}A_{ij}f_if_j  \\
 &=\frac{1}{2}( \sum\limits_{i=1}^{n}d_i (f_i)^2 - 2 \sum\limits_{i,j=1}^{n}A_{ij}f_if_j + \sum\limits_{j=1}^{n} d_j (f_j)^2) \\
 &= \frac{1}{2}\sum \limits_{i,j=1}^{n}A_{ij}(f_i-f_j)^2
\end{align*}
\end{proof}

\begin{proof}[Proof of Lemma \ref{TH1}]
Define a group of indicators $\mathbf{f}_j = (f_{1j}, f_{2j},\cdot\cdot\cdot,f_{nj})'$ for $k$ subgraph $R_j, ~ j =1,2,\cdot\cdot\cdot,k$ and $n$ observations $v_i, ~ i=1,2,\cdot\cdot\cdot,n$ as:
$$
f_{ij}= \begin{cases} 0& { v_i \notin R_j}\\ \sqrt{\frac{1}{\lvert R_j \lvert}}& { v_i \in R_j} \end{cases}
$$
and combine them together $\mathbf{F}^{n\times k}=[\mathbf{f}_1,\mathbf{f}_2,\cdot\cdot\cdot,\mathbf{f}_k]$, which satisfies $\mathbf{F}'\mathbf{F}=\mathbf{I}$. For simplicity, we assume $\lvert R_i\lvert \leq \lvert \bar{R}_i \lvert$ for all $i=1,2,\cdot\cdot\cdot,k$ here. Then,

\begin{align*}
\mathbf{f}_i'\mathbf{Lf}_i & = \mathbf{f}_i'(\mathbf{D}-\mathbf{W})\mathbf{f}_i \\
& =\frac{1}{2} \sum\limits_{m=1}\sum\limits_{n=1}w_{mn}(f_{im}-f_{in})^2 \\
& =\frac{1}{2} [\sum\limits_{m \in R_i, n \notin R_i}
A_{mn}(\frac{1}{\sqrt{\lvert R_i \lvert}} - 0)^2 +  \sum \limits_{m \notin R_i, n \in R_i}A_{mn}(0 - \frac{1}{\sqrt{\lvert R_i \lvert }} )^2] \\
& = \frac{1}{2} [ \sum\limits_{m \in R_i, n \notin R_i}A_{mn}\frac{1}{\lvert R_i \lvert} +  \sum\limits_{m \notin R_i, n \in R_i}A_{mn}\frac{1}{\lvert R_i\lvert}] \\
& = \frac{1}{2} [ l(R_i, \overline{R}_i) + l(\overline{R}_i, R_i)] \\
&= l(R_i, \overline{R}_i)=   \frac{1}{2}\frac{W(R_i,\bar{R_i})}{\lvert R_i\lvert}.
\end{align*}

Then we can simplify the loss of cut as

\begin{align*}
l(R_1,R_2,\cdot\cdot\cdot,R_k) &= \frac{1}{2}\sum\limits_{i=1}^{k}\frac{W(R_i, \overline{R}_i )}{\lvert R_i \rvert } = \sum\limits_{i=1}^{k}\mathbf{f}_i'\mathbf{L}\mathbf{f}_i\\
&= \sum\limits_{i=1}^{k}(\mathbf{F}'\mathbf{LF})_{ii} = Tr(\mathbf{F}^T\mathbf{LF}).
\end{align*}
Thus, to minimize the total cut loss is equivalent to
\begin{align}
\arg \min_{\mathbf{F}} Tr(\mathbf{F}'\mathbf{LF})  ~~ s.t. ~~ \mathbf{F}'\mathbf{F}=\mathbf{I}
\end{align}
\end{proof}

\section{The proof of Proposition \ref{CT}}

The proof of Proposition \ref{CT} shares the similar idea of the NDFS in \citet{Ld2012u}. To show the convergence of the optimization algorithm of the GOLFS, we first introduce two lemmas.

\begin{lemma}
\label{lemma1}
\citep{LS1999} \citep{LS2000} For a given function $f(h)$, if function $G(h,h')$ satisfies:
$$
G(h,h')\geq f(h);~G(h,h)=f(h)
$$
we call it an "auxiliary function". Then $f(h)$ is non-increasing under the update
$$
h^{t+1}=\arg \min_h G(h,h^t)
$$
\end{lemma}

\begin{lemma}
\label{lemma2} \citep{Nd2010}
$$
\sqrt{v}-\frac{v}{2\sqrt{v_t}} \leq \sqrt{v_t}-\frac{v_t}{2\sqrt{v_t}}
$$
\end{lemma}

When we update $\mathbf{F}$ holding $\mathbf{W}$ unchanged, we denotes the related parts in objective function as $f(\mathbf{F})$:
\begin{equation}
\label{ }
f(\mathbf{F})=Tr(\mathbf{F}'\mathbf{L}_1\mathbf{F})+\lambda Tr(\mathbf{F}'\mathbf{L}_0\mathbf{F})+Tr(\mathbf{F}'\mathbf{M}\mathbf{F})+\frac{\gamma}{2} \left\| \mathbf{F}'\mathbf{F}-\mathbf{I}_c \right\| _F^2
\end{equation}
When we update one element $F_{ij}$ of $\mathbf{F}$ holding the other elements constant, we can treat $f(\mathbf{F})$ as a function of $F_{ij}$ and rewrite it as $f_{ij}(F_{ij})=f(\mathbf{F})$. The first and second derivative of $f_{ij}(F_{ij})$ with respect to $F_{ij}$ is :
\begin{align}
\label{}
    f_{ij}'(F_{ij}) &= 2[\mathbf{L}_1\mathbf{F}+\lambda\mathbf{L}_0 \mathbf{F}+\mathbf{MF}+\gamma\mathbf{F}(\mathbf{F}'\mathbf{F}-\mathbf{I}_c)]_{ij}   \\
    f_{ij}''(F_{ij}) &= 2[\mathbf{L}_1+\lambda\mathbf{L}_0+\mathbf{M}+\gamma(3\mathbf{F}\mathbf{F}'-\mathbf{I}_c)]_{ij}
\end{align}
According to Taylor expansion for $f_{ij}(F_{ij})$ at point $F_{ij}^t$:
\begin{align}
\label{}
f_{ij}(F_{ij}) \approx f_{ij}(F_{ij}^t) + (F_{ij}-F_{ij}^t)f_{ij}'(F_{ij})+\frac{1}{2}(F_{ij}-F_{ij}^t)^2f_{ij}''(F_{ij})
\end{align}
Define
\begin{equation}
\label{ }
G(F_{ij},F_{ij}^t) = f_{ij}(F_{ij}^t)+f_{ij}'(F_{ij}^t)(F_{ij}-F^t_{ij})+K_{ij}(\mathbf{F}^t)(F_{ij}-F^t_{ij})^2
\end{equation}
where $K_{ij}(\mathbf{F}^t)=\frac{[\mathbf{L}_1\mathbf{F}^t+\lambda\mathbf{L}_0\mathbf{F}^t+\mathbf{MF}^t+\gamma\mathbf{F}^t(\mathbf{F}^t)'\mathbf{F}^t]_{ij}}{F_{ij}^t}$. It's easy to check the function $G(F_{ij},F_{ij}^t)$ satisfies:
\begin{equation}
\label{ }
G(F_{ij},F_{ij}^t) \geq f_{ij}(F_{ij}),~G(F_{ij},F_{ij})=f_{ij}(F_{ij})
\end{equation}
By lemma (\ref{lemma1}), $G(F_{ij},F_{ij}^t)$ is an auxiliary function and $f_{ij}(F_{ij})$ is non-increasing under the update $F_{ij}^{t+1}=\arg \min_{F_{ij}} G(F_{ij},F_{ij}^t)$, i.e. $f_{ij}(F^{t+1}_{ij})\leq G(F{ij}^{t+1},F{ij}^t) \leq f_{ij}(F_{ij}^t)$. With fixed $\mathbf{W}$, it's equivalent to $\mathcal{L}(\mathbf{F}^{t+1},\mathbf{W}^t) \leq \mathcal{L}(\mathbf{F}^{t},\mathbf{W}^t)$.

Therefore, we can derive the update of $F_{ij}$ by the first order condition of $\frac{\partial G(F_{ij},F_{ij}^t)}{\partial F_{ij}}=0$:
\begin{align}
\label{}
    f'_{ij}(F_{ij}^t)+\frac{2[\mathbf{L}_1\mathbf{F}^t+\lambda\mathbf{L}_0\mathbf{F}^t+\mathbf{MF}^t+\gamma\mathbf{F}^t(\mathbf{F}^t)'\mathbf{F}^t]_{ij}}{F_{ij}^t} (F_{ij}-F_{ij}^t) = 0   \\
   F_{ij}^{t+1}= \frac{(\gamma\mathbf{F}^t)_{ij}}{[\mathbf{L}_1\mathbf{F}^t+\lambda\mathbf{L}_0\mathbf{F}^t+\mathbf{MF}^t+\gamma \mathbf{F}^t (\mathbf{F}^t)'   \mathbf{F}^t]_{ij}} \mathbf{F}^t_{ij}
\end{align}
This coincide with the update of $F_{ij}$ in the optimization algorithm of GLOFS.

When we update $\mathbf{W}$ with fixed $\mathbf{F}$, $\mathbf{W}=(\mathbf{X}'\mathbf{X}+\beta\mathbf{D})^{-1}\mathbf{X}'\mathbf{F}$ is the solution of following question:
\begin{equation}
\min_{\mathbf{W}} \left\| \mathbf{X}\mathbf{W}-\mathbf{F}\right\| _F^2+\beta Tr(\mathbf{W}'\mathbf{DW})
\end{equation}
Thus,
\begin{equation}
\mathbf{W}^{t+1} = \arg\min_{\mathbf{W}} \left\| \mathbf{X}\mathbf{W}-\mathbf{F}^t\right\|_F^2+\beta Tr(\mathbf{W}'\mathbf{D}^t \mathbf{W})
\end{equation}
Substituting $\mathbf{W}^{t+1}$ into the objective function and using lemma \ref{lemma2}, we can derive the following inequality with similar procedure in \citep{Ld2012u}:
\begin{equation}
\left\|\mathbf{X}\mathbf{W}^{t+1}-\mathbf{F}^{t}\right\|_F^2+\beta \left\|\mathbf{W}^{t+1}\right\|_{2,1}
\leq
\left\|\mathbf{X}\mathbf{W}^{t}-\mathbf{F}^{t}\right\|_F^2+\beta \left\|\mathbf{W}^{t}\right\|_{2,1}
\end{equation}
with positive $\beta$. This show $\mathcal{L}(\mathbf{F}^{t+1}, \mathbf{W}^{t+1}) \leq \mathcal{L}(\mathbf{F}^{t+1},\mathbf{W}^t)$. Together with $\mathcal{L}(\mathbf{F}^{t+1},\mathbf{W}^t) \leq \mathcal{L}(\mathbf{F}^{t}  \mathbf{W}^t) $ proved previously, we finally get
\begin{equation}
\label{ }
\mathcal{L}(\mathbf{F}^{t+1}, \mathbf{W}^{t+1}) \leq \mathcal{L}(\mathbf{F}^{t+1},\mathbf{W}^t) \leq \mathcal{L}(\mathbf{F}^{t},\mathbf{W}^t)
\end{equation}

That is, the objective function $\mathcal{L}(\mathbf{F}, \mathbf{W})$ monotonically decreases during the updating process. Therefore, we can obtain the local optimal $\mathbf{W}$ through the alternative and iterative optimization algorithm.

\section{Supporting information}

\subsection{The sensitivity of tuning parameters}
In this section, we study the sensitivity of tuning parameters
for the GOLFS algorithm. For simplicity, we only report the experimental results for simulation example 1 as an illustration. We evaluate Clustering Accuracy (ACC), Normalized Mutual Information (NMI), and Adjusted Rand Index (ARI) over different values of tuning parameters and the results are present in Figure \ref{SOPS}. The darkness and heights of bars indicate the values of evaluations. The results show the performance of GOLFS is not so sensitive to tuning parameters $\alpha$ and $\beta$.  In practice, when the label is unknown, we can select tuning parameters by optimizing some internal measurements, such as average inner-group distance (AID), Calinski Harabasz score (CHS), and Davies Bouldin index (DBI).

\begin{figure}[ht]
\begin{center}
\includegraphics[width=4.5in]{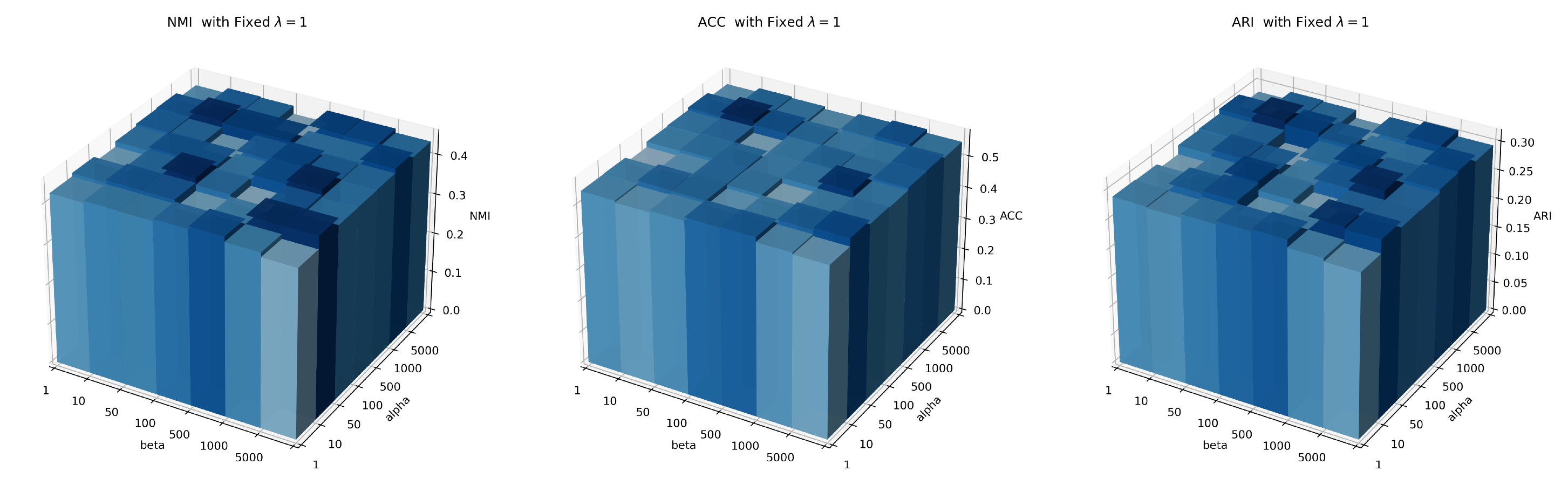}
\caption{The sensitivity of tuning parameters: simulation example 1}
\label{SOPS}
\end{center}
\end{figure}
 
\subsection{The word cloud of the clusters in empirical analysis}

For Example 3 in the empirical analysis, the GOLFS is used to select the true discriminative keywords and then the articles are clustered by K-means clustering based on the selected words. The bar charts in Section 4 only show the top 10 features selected in each cluster. Here we use the word cloud to show all the keywords selected for each cluster additionally. We can also find the keywords selected for each cluster are discriminative and reasonable.

\begin{figure}[H]
\begin{center}
  \subfigure[cluster 1]{
    \includegraphics[height=0.5in]{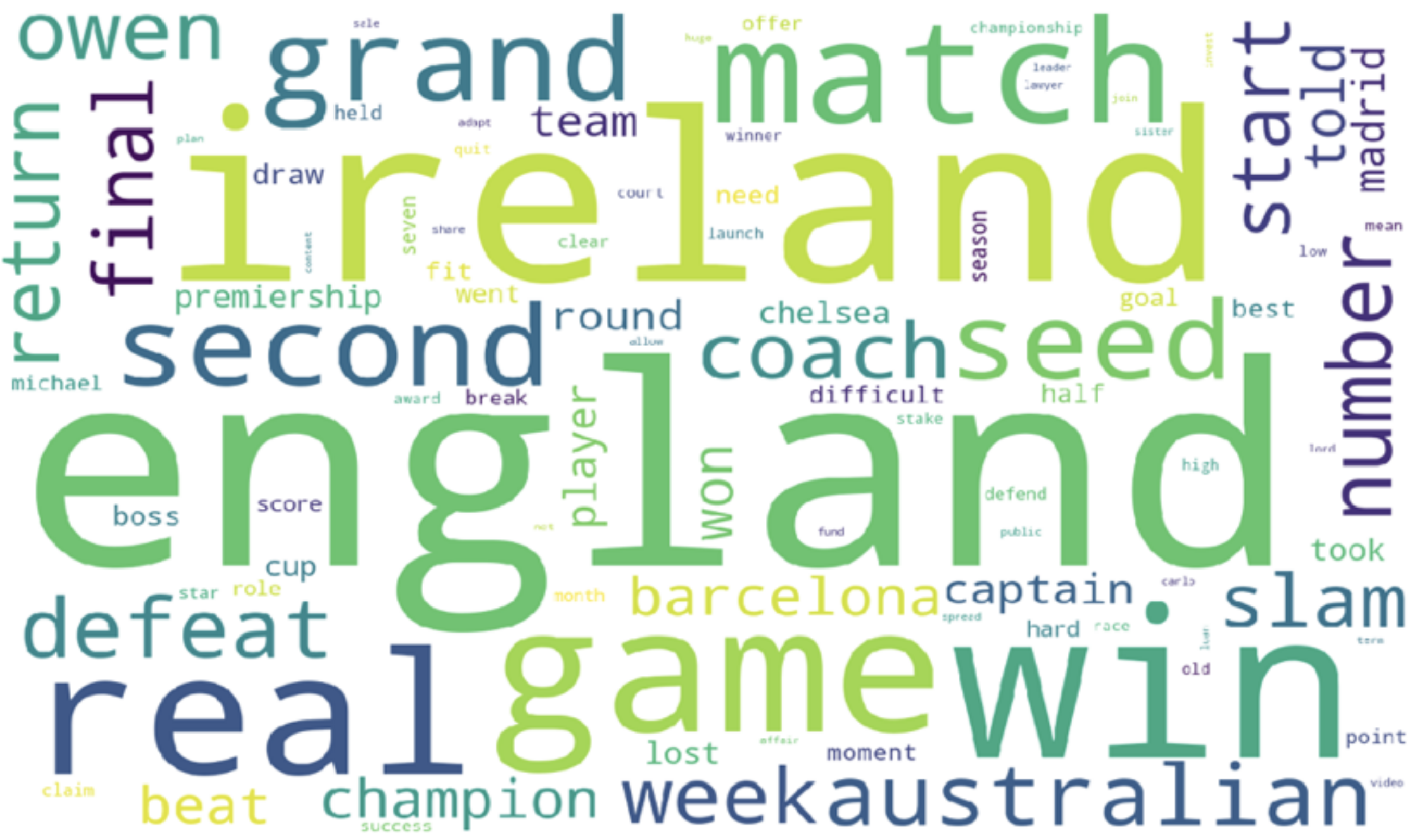}
  }
  \subfigure[cluster 2]{
    \includegraphics[height=0.5in]{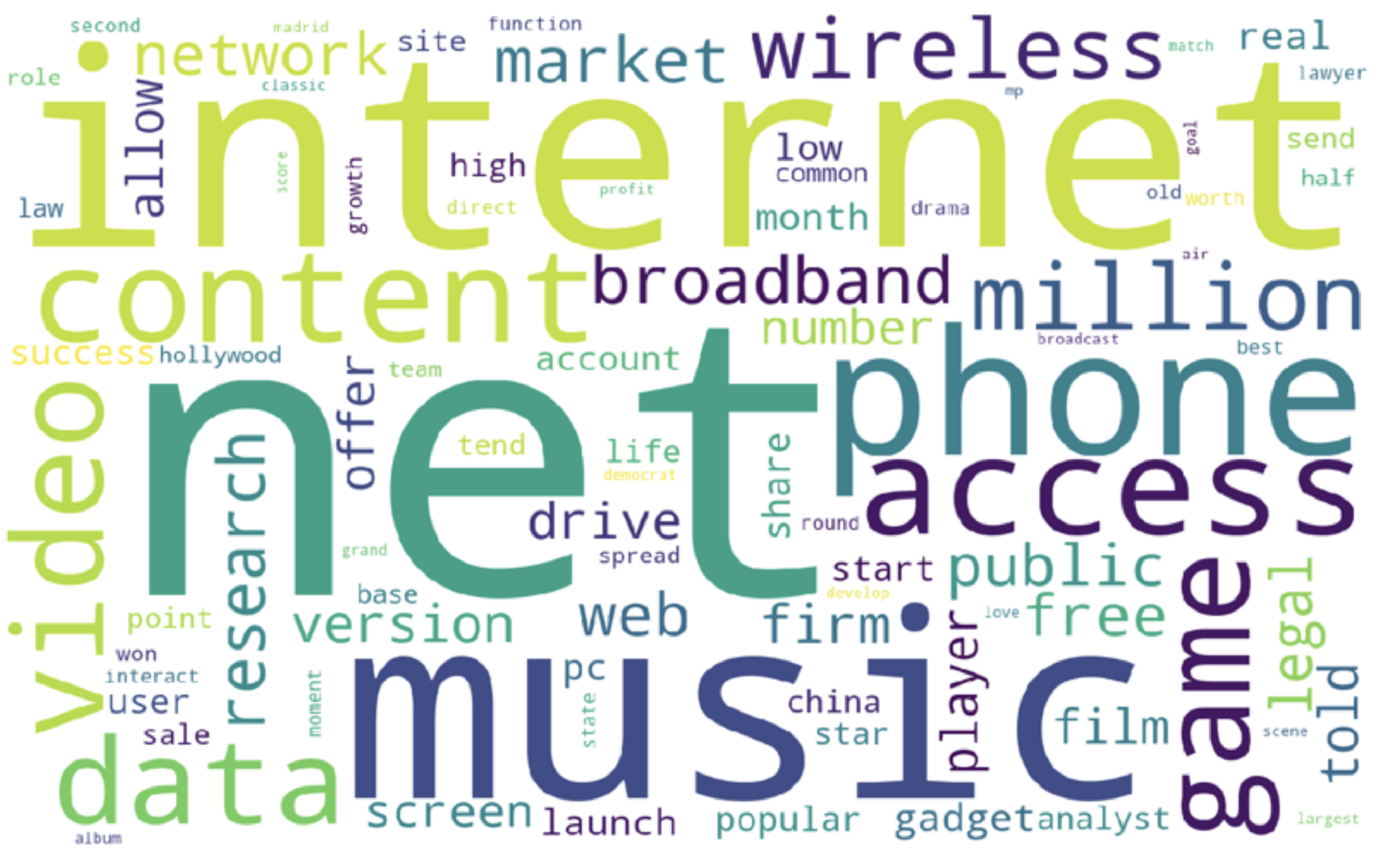}
  }
  \subfigure[cluster 3]{
    \includegraphics[height=0.5in]{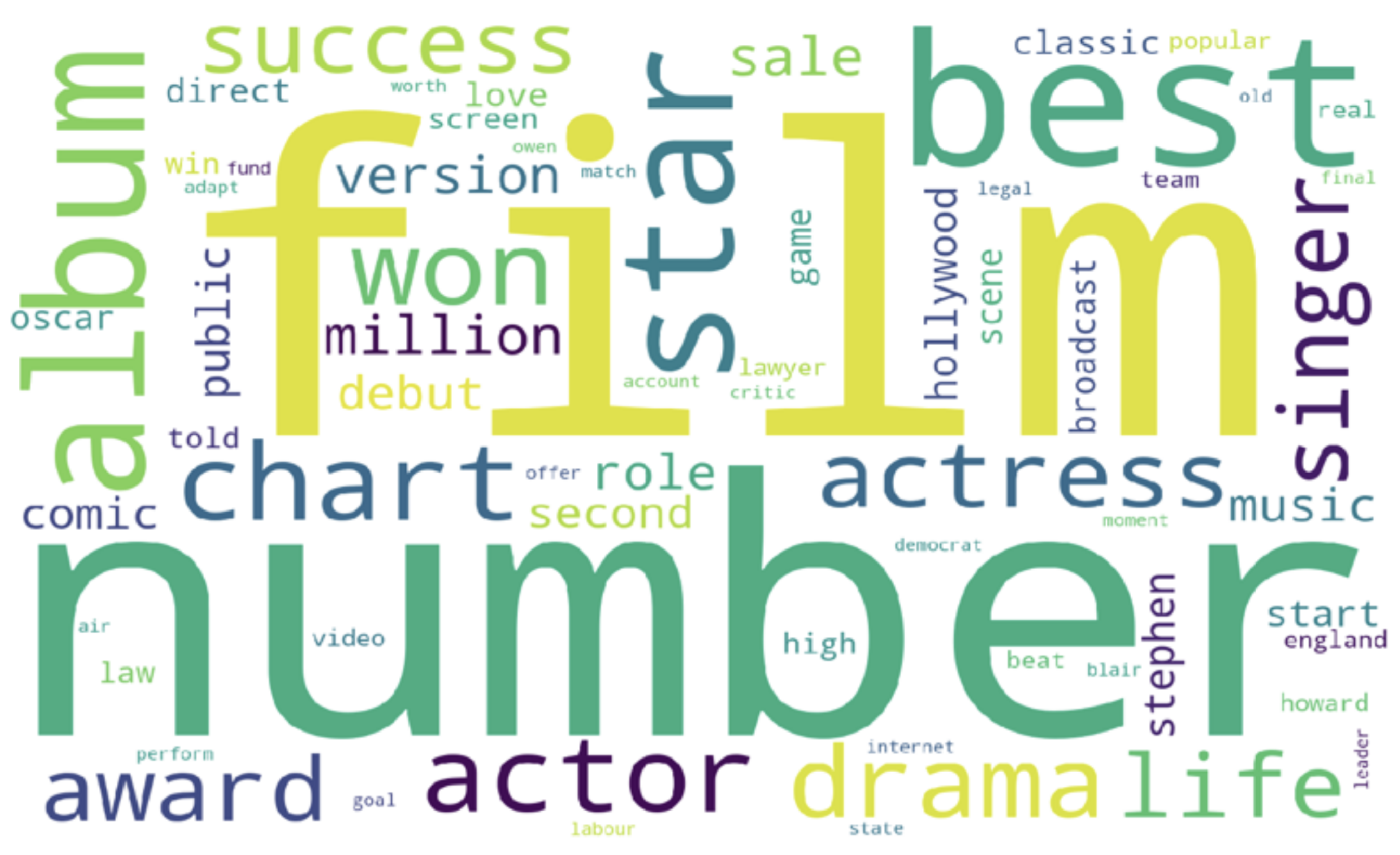}
  }
  \subfigure[cluster 4]{
    \includegraphics[height=0.5in]{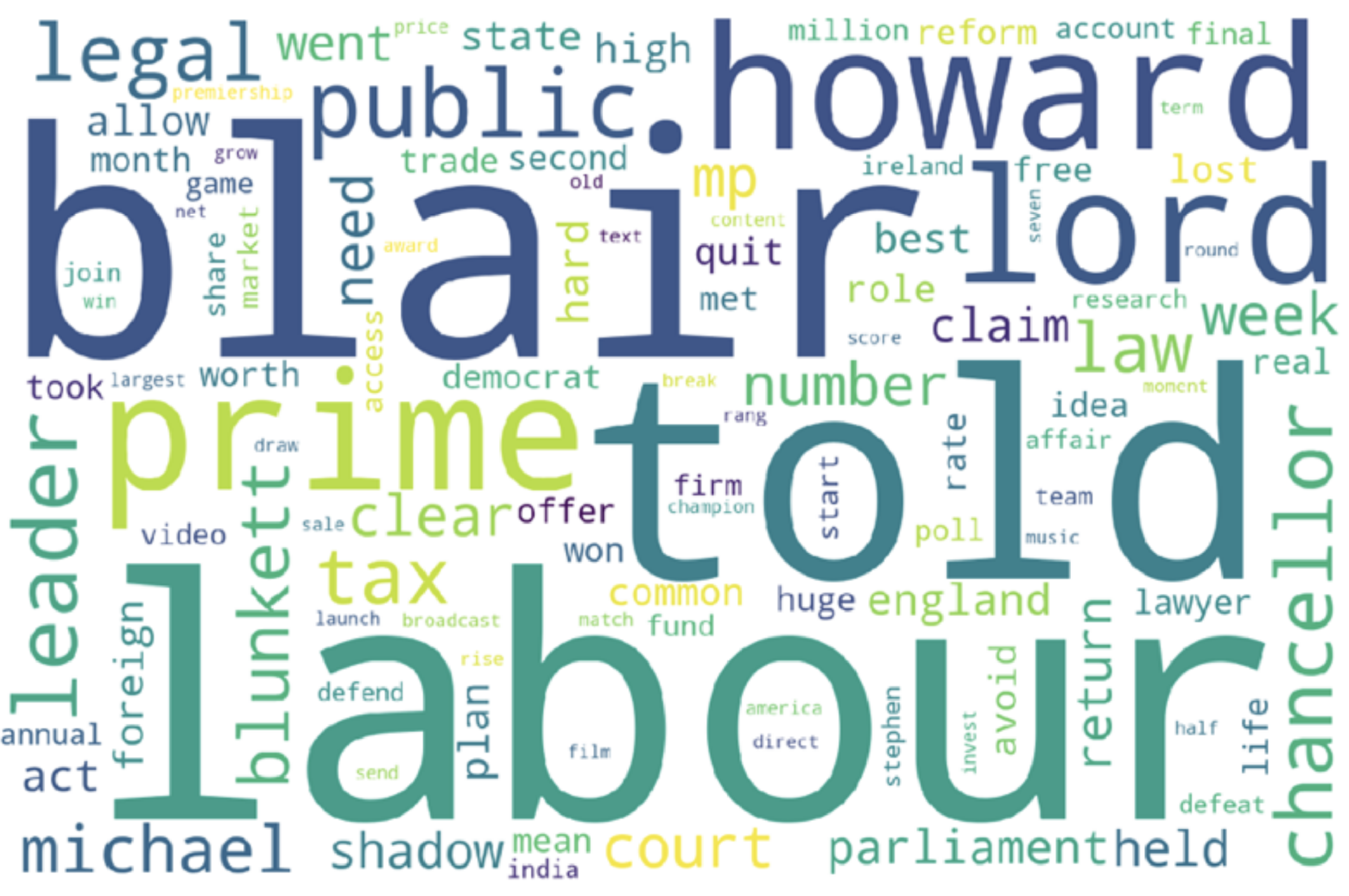}
  }
  \subfigure[cluster 5]{
    \includegraphics[height=0.5in]{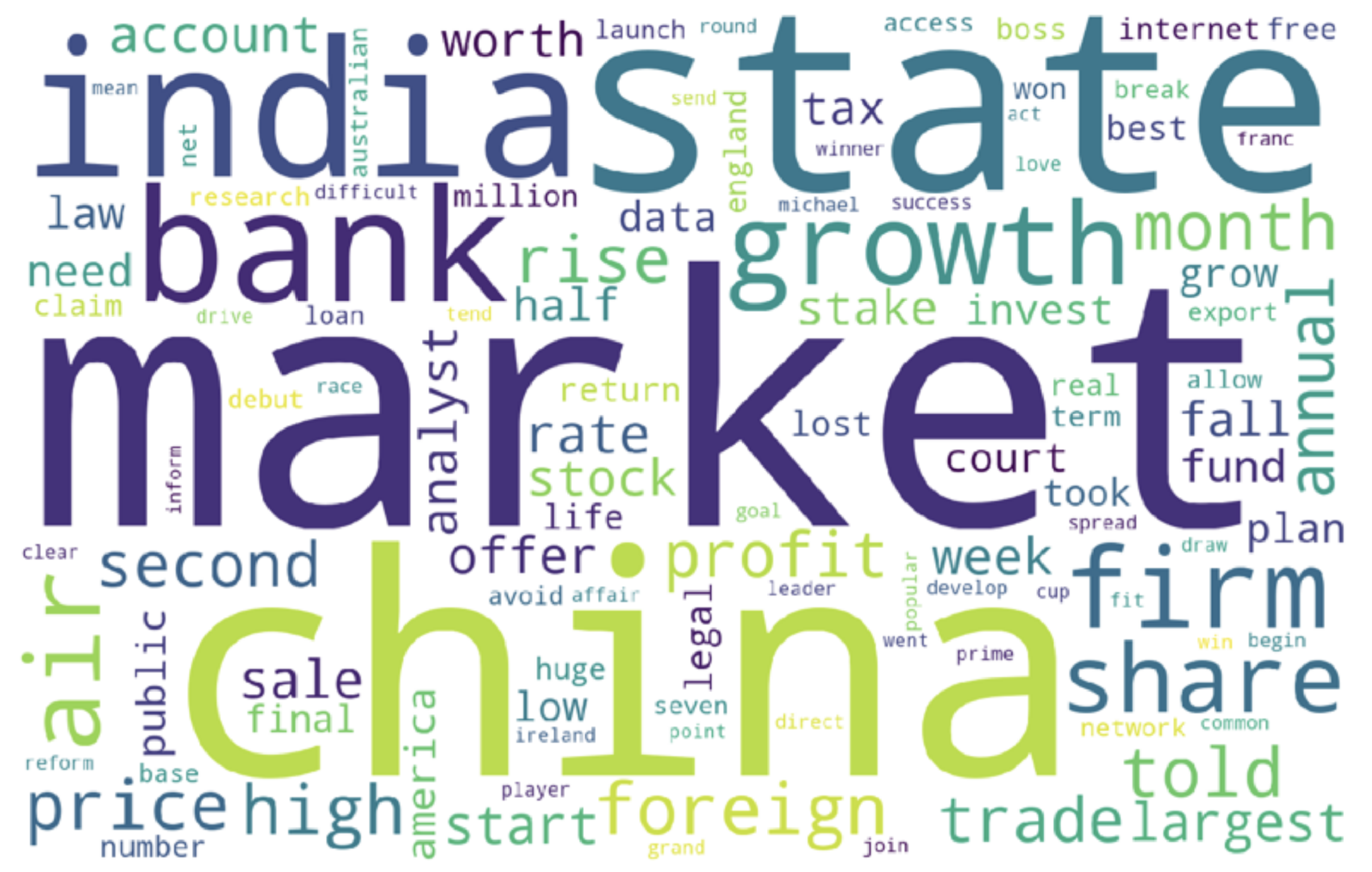}
  }
\end{center}
\caption{Word clouds of the top keywords selected by the GOLFS in 5 clusters.}
\label{BBCfig}
\end{figure}

\subsection{The performances of the clustering algorithms}

In this section, Besides the K-Means method, we additionally consider two different classic clustering methods, hierarchical clustering (HC), and spectral clustering (SC), after selecting an important subset of features. We consider the following six evaluation criteria. Like the simulations and real data analysis in the main text, we consider Clustering Accuracy (ACC), Normalized Mutual Information (NMI), and Adjusted Rand Index (ARI) to measure the similarity between the clustering results based on selected features and the true labels. Without the information about the true labels, we additionally measure the results using three internal measurements: average inner-group distance (AID), Calinski Harabasz score (CHS), and Davies Bouldin index (DBI). Lower values of AID and DBI indicate better clustering results.
The comparison results for simulation examples 1-2 are summarized in Tables C1-C2. We can observe that the GOLFS algorithm can obtain better performances than the other existing feature selection methods no matter which clustering method is used for the selected subset of features.

\begin{table}[h]
\centering
\footnotesize
\caption{Performance comparison among different unsupervised feature selection methods and clustering methods in Simulation Example 1}
\label{SSSS}
\begin{tabular}{@{}cccccccc@{}}
\toprule
\textbf{Clustering}     & \textbf{Method} & \textbf{NMI(\%)}   & \textbf{ACC($10^{-2}$)}   & \textbf{ARI($10^{-2}$)}   & \textbf{AID} & \textbf{CHS} & \textbf{DBI} \\ \midrule
\multirow{7}{*}{KMeans} & Benchmark       & 10.54          & 31.62          & 3.30           & 21632.16                     & 1.69             & 5.40               \\
                        & DIPTEST         & 5.90           & 29.81          & 0.27           & 488.55                       & 5.38             & 2.97               \\
                        & MCFS            & 5.75           & 30.37          & 0.25           & 2.09                         & 5.89             & 2.89               \\
                        & NDFS            & 37.80          & 45.32          & 17.43          & 0.30                         & 583.93           & 0.63               \\
                        & RSR             & 18.74          & 38.00          & 8.24           & 1.92                         & 5.15             & 8.85               \\
                        & UDFS            & 6.40           & 31.08          & 0.71           & 2.09                         & 5.93             & 1.21               \\
                        & \textbf{GOLFS}  & \textbf{38.41} & \textbf{45.92} & \textbf{17.85} & \textbf{0.30}                & \textbf{584.85}  & \textbf{0.56}      \\ \midrule
\multirow{7}{*}{HC}     & Benchmark       & 14.26          & 31.62          & 4.46           & 21644.40                     & 1.70             & 3.74               \\
                        & DIPTEST         & 8.02           & 23.92          & -0.03          & 489.65                       & 2.30             & 1.41               \\
                        & MCFS            & 7.06           & 23.71          & -0.06          & 2.09                         & 2.54             & 1.21               \\
                        & NDFS            & 26.89          & 42.56          & 8.43           & 0.30                         & 583.93           & 0.63               \\
                        & RSR             & 17.98          & 33.15          & 5.61           & 1.92                         & 5.15             & 8.85               \\
                        & UDFS            & 7.06           & 23.71          & -0.06          & 2.09                         & 5.93             & 1.21               \\
                        & \textbf{GOLFS}  & \textbf{26.89} & \textbf{42.96} & \textbf{8.43}  & \textbf{0.30}                & \textbf{584.85}  & \textbf{0.56}      \\ \midrule
\multirow{7}{*}{SC}     & Benchmark       & 10.38          & 33.62          & 3.56           & 21646.07                     & 1.55             & 5.54               \\
                        & DIPTEST         & 6.21           & 29.69          & 0.56           & 488.22                       & 4.70             & 2.93               \\
                        & MCFS            & 5.53           & 30.92          & 0.11           & 2.09                         & 5.69             & 2.86               \\
                        & NDFS            & 29.86          & 41.81          & 12.71          & 0.30                         & 583.93           & 0.63               \\
                        & RSR             & 12.32          & 34.54          & 4.90           & 1.92                         & 5.15             & 8.85               \\
                        & UDFS            & 5.53           & 30.92          & 0.11           & 2.09                         & 5.93             & 1.21               \\
                        & \textbf{GOLFS}  & \textbf{30.01} & \textbf{41.90} & \textbf{12.78} & \textbf{0.30}                & \textbf{584.85}  & \textbf{0.56}      \\ \bottomrule
\end{tabular}
\begin{tablenotes}
{\footnotesize \raggedleft Note: ACC denotes the clustering accuracy, NMI denotes the normalized mutual information, ARI denotes the adjusted Rand index, AID denotes the average inner-group distance, CHS represents Calinski Harabasz score, DBI denotes the Davies Bouldin index, HC denotes the hierarchical clustering and SC denotes the spectral clustering.}
\end{tablenotes}
\end{table}

\begin{table}[h]
\centering
\footnotesize
\caption{Performance comparison among different unsupervised feature selection methods and clustering methods in Simulation Example 2}
\label{SSSS2}
\begin{tabular}{@{}cccccccc@{}}
\toprule
\textbf{Clustering}     & \textbf{Method} & \textbf{NMI($10^{-2}$)}   & \textbf{ACC(\%)}   & \textbf{ARI($10^{-2}$)}  & \textbf{AID} & \textbf{CHS} & \textbf{DBI} \\ \midrule
\multirow{7}{*}{KMeans}   & Benchmark       & 30.77          & 43.00          & 14.36          & 11183.01                     & 1.37             & 3.87               \\
                    & DIPTEST         & 14.72          & 35.21          & 1.70           & 239.02                       & 4.18             & 2.29               \\
                    & MCFS            & 19.21          & 37.47          & 5.39           & 0.96                         & 4.18             & 2.38               \\
                    & NDFS            & 51.66          & 53.67          & 30.36          & 0.20                         & 142.53           & 0.87               \\
                    & RSR             & 49.52          & 55.32          & 29.95          & 1.75                         & 6.81             & 10.00              \\
                    & UDFS            & 20.67          & 38.91          & 6.50           & 0.95                         & 4.26             & 1.26               \\
                    & \textbf{GOLFS}  & \textbf{52.45} & \textbf{55.13} & \textbf{31.50} & \textbf{0.19}                & \textbf{144.88}  & \textbf{0.52}      \\ \midrule
\multirow{7}{*}{HC} & Benchmark       & 35.28          & 42.13          & 19.71          & 11185.56                     & 1.41             & 3.10               \\
                    & DIPTEST         & 14.67          & 29.13          & 0.46           & 239.03                       & 2.72             & 1.40               \\
                    & MCFS            & 18.33          & 30.50          & 3.02           & 0.96                         & 2.75             & 1.26               \\
                    & NDFS            & 52.62          & 49.03          & 29.58          & 0.20                         & 142.53           & 0.87               \\
                    & RSR             & 52.68          & 49.97          & 31.46          & 1.75                         & 6.81             & 10.00              \\
                    & UDFS            & 18.33          & 30.50          & 3.02           & 0.95                         & 4.26             & 1.26               \\
                    & \textbf{GOLFS}  & \textbf{52.81} & \textbf{49.19} & \textbf{29.92} & \textbf{0.19}                & \textbf{144.88}  & \textbf{0.52}      \\ \midrule
\multirow{7}{*}{SC} & Benchmark       & 22.16          & 38.25          & 6.61           & 11184.38                     & 1.30             & 4.02               \\
                    & DIPTEST         & 13.93          & 35.69          & 1.10           & 238.88                       & 3.88             & 2.31               \\
                    & MCFS            & 16.87          & 36.50          & 3.49           & 0.96                         & 4.03             & 2.34               \\
                    & NDFS            & 38.85          & 48.66          & 23.83          & 0.20                         & 142.53           & 0.87               \\
                    & RSR             & 33.63          & 46.97          & 17.89          & 1.75                         & 6.81             & 10.00              \\
                    & UDFS            & 16.87          & 36.50          & 3.49           & 0.95                         & 4.26             & 1.26               \\
                    & \textbf{GOLFS}  & \textbf{39.82} & \textbf{49.44} & \textbf{24.43} & \textbf{0.19}                & \textbf{144.88}  & \textbf{0.52}      \\ \bottomrule
\end{tabular}
\begin{tablenotes}
{\footnotesize   \raggedleft Note: ACC denotes the clustering accuracy, NMI denotes the normalized mutual information, ARI denotes the adjusted Rand index, AID denotes the average inner-group distance, CHS represents Calinski Harabasz score, DBI denotes the Davies Bouldin index, HC denotes the hierarchical clustering and SC denotes the spectral clustering.}
\end{tablenotes}
\end{table}

\end{appendices}
\end{document}